\documentclass[10pt,draftclsnofoot,onecolumn]{IEEEtran}

\usepackage{color,graphicx,amsmath,amssymb,amsthm,cite,graphics}
\usepackage{amssymb}
\usepackage{float}
\usepackage{tikz}
\usetikzlibrary{trees}
\usepackage{algpseudocode}
\usepackage{enumerate}
\usepackage{algpseudocode}
\usepackage{algorithm} 
\usepackage{stmaryrd}
\usepackage{amsthm}
\usepackage{tikz}
\usetikzlibrary{arrows}
\usepackage{subcaption}

\tikzset{
    treenode/.style   = {   align=center, 
                            inner sep=0pt, 
                            text centered,
                            font=\sffamily },
    selNode/.style  = {     treenode, 
                            circle, 
                            fill=black!20, 
                            font=\sffamily\bfseries, 
                            draw=black, 
                            text width=2.0em,
                            very thick },
    blankNode/.style    = { treenode,
                            circle, 
                            black, 
                            draw=black, 
                            text width=2.0em }
} 

\newtheorem{theorem}{Theorem}[section]
\newtheorem{corollary}{Corollary}[theorem]
\newtheorem{lemma}[theorem]{Lemma}

\author{ Vishnu Raj \hspace{16pt} Sheetal Kalyani \\
    \hspace{-0 cm}Department of Electrical Engineering \\ Indian Institute of Technology, Madras \\
    Chennai, India 600 036\\
    \texttt{\{ee14d213,skalyani\}@ee.iitm.ac.in}
}

\title{An aggregating strategy for shifting experts in discrete sequence prediction}

\begin{document}
    \maketitle

    \begin{abstract}
        We study how we can adapt a predictor to a non-stationary environment with advices from 
        multiple experts. We study the problem under complete feedback when the best expert
        changes over time from a decision theoretic point of view. Proposed algorithm
        is based on popular exponential weighing method with exponential discounting.
        We provide theoretical results bounding regret under the exponential discounting
        setting. Upper bound on regret is derived for finite time horizon problem. Numerical
        verification of different real life datasets are provided to show the utility of proposed
        algorithm.
    \end{abstract}


\section{Introduction}    \label{sec:intro}
        Online prediction is a widely studied topic with applications including user
        activity prediction, webpage requests prediction, location prediction in wireless mobile networks etc., 
        It deals with the problem of predicting future symbols by observing an online stream
        of symbols in the sequence. Multiple approaches are proposed from diverse fields 
        such as information theory (See \cite{Feder1992}, \cite{Willems1995},
        \cite{Dekel2009}), machine learning (See \cite{Cover1977a}, \cite{Bialek2001},
        \cite{Dietterich}, \cite{Begleiter2004}) etc. One of the widely successful approach for 
        online prediction is to assume \textit{Markov} property while modeling the sequence.
        By assuming that a  finite history of past observations will be enough for predicting
        the future symbols, Markovian models make predictions based on the observed history. 
        This finite history of symbols is referred to as \textit{context} and the length 
        of context is the \textit{order} of Markov model.

        Creating a Markovian model based predictor comprises of two steps - building a frequency
        table for each context based on past observations and devising a method to make 
        predictions from the table. A tree structure is one of the popular methods for storing 
        the frequency counts. These frequency trees can be stored as special data structures known
        as tries (See \cite{Fredkin1960}). In a trie, each node can represent a context from the frequency
        table and prediction is made based on the symbol distribution at each node.
        One of the main challenges while constructing the frequency tree is the loss of
        information at cross phrase boundaries. Rate of convergence to optimal predictability
        is another challenge during frequency tree construction. By updating the count at
        each node (context) based on the observed symbol, a trie structure can converge to 
        the Markovian model of underlying sequence generator with enough data. Once such a 
        tree is built, prediction task merely becomes selecting the most probable symbol
        based on the context. Challenge lies in how to predict while the tree is still being built.
        A few notable works in this area include Prediction by Partial Matching (See \cite{Cleary1984},
        \cite{Cleary1995}), Context Tree Weighting (See \cite{Willems1995}), Probabilistic Suffix Trees
        (See \cite{Ron1996}), Compact Prediction Trees (See \cite{Gueniche}, \cite{Gueniche2015}) etc.
        
        \cite{Cleary1984} proposed \textit{Prediction by Partial Matching} (PPM) 
        initially for data compression as a method for encoding symbols in a sequence. Since
        predictability and compressibility are related (See \cite{Feder1992}), PPM can also be used
        in predicting future symbols from discrete sequences. By combining predictions from 
        multiple sub-contexts from the tree, PPM attempts to model the symbol probability as 
        combination of multiple sub-contexts within the given context. 
        \cite{Das2002} showed a successful application of PPM on frequency trie based on LeZi update for the 
        task of inhabitant action prediction in a smart home environment. Since sequences
        are broken into sub-sequences and update of trie is performed based on these \textit{phrases},
        information about the relationship between the symbols at both ends of sub-sequences are
        lost. This loss at \textit{cross-phrase boundary}, along with the complexity of sequence
        and number of symbols in the alphabet contribute together to the slow convergence of 
        frequency tree.         
        \cite{Gopalratnam2003} proposed Active LeZi as a method
        to optimally predict the symbol without losing the information at cross-phrase
        boundaries. To address the problem of slow convergence, they used a sliding window over the
        phrases with variable length.
        For the 
        prediction phase, Active LeZi employs PPM. One of the useful measures in analyzing 
        convergence of a predictor is \textit{FS Predictability}. Finite State (FS) Predictability
        of an infinite sequence is defined as the minimum fraction of prediction errors that 
        can be made by any finite state predictor (See \cite{Feder1992}). Active LeZi is able to build 
        predictors which are able to  converge to FS Predictability at the rate $O(\sqrt{n})$. 
        
        
        Another challenge in building frequency trees is to decide on the maximum depth of the tree.
        This is an important issue from a practical point of view, as it is not possible
        to enjoy an infinite memory requirement. \cite{Korodi2008} proposed 
        a finite bounded length context tree approach and used distribution entropy for order 
        estimation focusing on data compression. Even though the proposed method did not perform 
        better than other PPM improvements, it showed that by properly bounding the context length,
        it is possible to obtain a performance comparable to other methods.
        
        The success of Markovian predictors depends on how strong the assumption of stationarity 
        in underlying generating process is. If the sequence is non-stationary, then the 
        predictor needs to adapt. \cite{Pulliyakode2014} proposed an adaptive model 
        combining algorithm
        which is loosely based on Active Lezi and has empirically shown to achieve
        better results than conventional PPM methods. Following the work of \cite{Korodi2008},
        they employed a fixed length sliding window for constructing the trie, rather than 
        using a variable length sliding window over cross phrase boundaries.
        PPM algorithm with multiple context lengths are ran on this trie and the algorithm
        weighs each model based on the past performance. These weights are later used to
        combine the predictions from each model. But no theoretical guarantees on the performance
        of predictor is provided.
        
        Prediction by combining multiple experts is a long studied problem. 
        Assuming that the decision maker (predictor) has complete
        knowledge about all the past decisions and performance of experts, the goal
        is to perform as good as the best expert in the pool. The method of using
        advices from multiple experts is introduced by \cite{Vovk1990} and is generalized
        as a strategies aggregating framework by \cite{Littlestone1994}. \cite{Freund1997}
        gave a decision theoretic generalization to the problem. By adopting
        a multiplicative update of weight parameter, they were able to produce algorithms performing
        almost as well as the best expert in the pool, within a loss of $O(\sqrt{N \ln K})$, 
        where $N$ denotes time-steps and $K$ is the number of experts in the pool.
        
        Given a training set of sequences, \cite{Eban2012} describes a procedure for learning a set of experts that will work on online prediction.
        On this pool of experts, they run traditional setting of statistical learning to produce
        a model whose expected loss is as small as possible over an online sequence. They achieve this
        by first minimizing training loss on the experts and then minimizing hindsight loss in
        online prediction for the learning model. \cite{Cortes2014} introduces a series of learning 
        algorithms for designing
        accurate ensembles of structured prediction task. Their goal is, given a set of labeled
        training examples, exploit sub structures present in the problem domain to design experts 
        and combine these experts to form an accurate ensemble. Here the experts are trained on a
        set of labeled samples and the ensemble algorithm is trained on a distinct set of samples. This
         ensemble is then used to predict labels for a 
        given sequence of labels.
        
        In both the cases above (\cite{Eban2012} and \cite{Cortes2014}), the models are first trained on a
        dataset that is considered to be uniformly sampled from the problem domain and then 
        prediction is performed online.
        We are in search of methods which does not require pre-training as the aforementioned methods, but
        are able to combine multiple experts on an online manner to produce better results.
        Our goal considerably varies from the above methods as we want neither our experts nor
        the combining forecaster to be pre-trained. Our goal is to train the experts and the final forecaster
        online while they are expected to make predictions. This pose challenges of experiencing
        a greater loss during the initial stages of prediction. 
        
        By combining Mixing Past Posteriors (MPP) (see \cite{Bousquet2003}) and AdaHedge (see \cite{DeRooij2014}),
        \cite{Vladimir2017} proposed an online aggregation algorithm for the problem of shifting experts.
        By using the adaptive learning rate of AdaHedge, they modified MPP and obtained regret bounds of 
        signed unbounded losses under adversarial setting. Empirical results provided show that the
        modified algorithm outperform AdaHedge in both synthetic and real data, even when the losses of
        experts are volatile.
        
        Motivated from Decision Theoretic Online Learning view of combining multiple experts and 
        Information Theoretic techniques for discrete sequence prediction, this paper propose a 
        discrete sequence predictor that trains online and adaptively adjusts to the changes in model.
        By applying exponential filtering over the past performance of experts, we present a modified
        version of HEDGE algorithm. We also prove the convergence of the model to FS Predictability
        (under stationary assumptions) and obtain an upper bound on the regret of the algorithm. 
        
        Section \ref{sec:problem_definition} formally introduces
        the problem along with the mathematical notations used in this paper. In 
        Section \ref{sec:model_cons}, we discuss the method for constructing the pool of experts
        online and prove the optimal convergence rate. Section \ref{sec:ada_prediction} introduces 
        proposed algorithm and in Section \ref{sec:regret_analysis} we derive the 
        rate of convergence to the best predictor. Experiments conducted
        to validate the proposed method are included in Section \ref{sec:experimental_results}.
    
    \section{Problem Formulation}    \label{sec:problem_definition}
        Let $\mathcal{S}$ denotes the symbol space alphabet. We assume there exists a source which emits
        a symbol $s[n] \in \mathcal{S}$ at discrete time instant $n$. We want to create a predictor, who
        observes all the symbols emitted from the source till time instant $n$, $s_{1:n} = (s[1],s[2],
        \ldots,s[n])$, and predicts the next
        symbol $\hat{s}[n+1]$. We also assume there exists a rewarding mechanism which, after observing
        the actual symbol $s[n+1]$ at time instant $n+1$, will appropriately reward the predictor.
        In generalized online method for prediction, the predictor cannot be assumed to have the knowledge 
        about the sequences and the symbols in the sequences. Hence,the predictor will 
        only have information about the symbols it has seen so far. This subset of symbols from the 
        alphabet constitutes the decision space for the
        predictor. Let $\mathcal{D}_{n}$ denotes the decision space of the predictor at time instant $n$
        with $\mathcal{D}_{n} \subseteq \mathcal{S}$. By construction, $\mathcal{D}_{n} \triangleq
        \{s_{1:n}\}$ where $\{\cdot\}$ is the set operator which returns the unique members in the input.
        
        Let $\mathcal{K}$ denotes the set of experts available to the predictor
        with $|\mathcal{K}| = K < \infty$, where $|\cdot|$ stands for cardinality of the set. 
        The pool of predictors we consider in this paper are Markovian models of order $k$ with 
        $k \in \{0,\ldots,K-1\}$. Let 
        $\hat{s}^{(k)}[n] \in \mathcal{D}_{n}^{(k)}$ be the prediction from $k^{th}$ predictor for time instant 
        $n$. Instantaneous loss of $k^{th}$ predictor is defined as $l^{(k)}[n] = \mathcal{L}(s[n],\hat{s}^{(k)}[n])$,
        $\mathcal{L}(\cdot)$ is the loss function. Define the cumulative loss incurred by $k^{th}$ predictor after $N$ time steps as
        \begin{align}    \label{eqn:discExpertLoss}
            L_{k,N}(\gamma) &= \sum \limits_{n=1}^{N} \gamma^{N-n} l^{(k)}[n],
        \end{align}
        where $ 0 < \gamma \leq 1$ is the discounting factor.
                
        The predictor maintains a
        probability distribution $\textbf{p}[n] \in [0,1]^K$ over the pool of experts. Defining the
        instantaneous loss of predictor as $l[n] = \sum \limits_{k=1}^{K} p^{(k)}[n] \cdot
        l^{(k)}[n]$, the discounted cumulative loss of predictor till time instant $N$ can be
        written as 
        \begin{align}  \label{eqn:predictorLoss}
            L_{N}(\gamma) 
                  = \sum \limits_{n=1}^{N} \gamma^{N-n} l[n] 
                  = \sum \limits_{n=1}^{N} \sum \limits_{k=1}^{K} \gamma^{N-n} p^{(k)}[n] \cdot l^{(k)}[n]
                  = \sum \limits_{n=1}^{N} \gamma^{N-n} \textbf{p}[n] \cdot \textbf{l}[n]
        \end{align}
        where $\textbf{l}[n]$ is the vector of instantaneous loss functions of all the experts at instant
        $n$.
        
        Now, objective of the prediction algorithm to perform as good as the best expert
        in the pool can be represented as
        \begin{align}    \label{exp:objective}
            \underset{\textbf{p}}{\min} \left\{ L_N(\gamma) - \underset{k}{\min}\: L_{k,N}(\gamma) \right\}
        \end{align}
            
        In this paper, we deal with a pool of predictors that share a common decision
        space which is based on only the observed symbols from a finite cardinality symbol
        space. Hence, $\mathcal{D}_{n}^{(k)} = \mathcal{D}_{n}$ $\forall$ $k \in \mathcal{K}$.
        The experts we consider for this problem setup are finite context length PPM predictors which are
        Markovian predictors with different depth levels.
    
    \section{Model Construction}    \label{sec:model_cons}
        This section introduces creating the pool of experts. These experts will be used in the second stage
        by the adaptive predictor to make final predictions. 
        It is empirically shown by \cite{Korodi2008} that predictor will not incur a 
        remarkable loss by bounding the depth. Following this observation, we create experts who are
        fixed context $K^{th}$-order Markov Models and then apply PPM approach to make predictions.
        Our trie building procedure is detailed in Algorithm \ref{alg:KOM-PPM}. PPM is used to calculate the 
        probability of each symbol from the model and the prediction is made as,
        \begin{align}     \label{eqn:expertPrediction}
            d^{(k)}[n] = \underset{d}{\arg \max} P(d|s[n-1],s[n-2],\ldots,s[n-k])
        \end{align}
        where $d$ is a symbol in the alphabet captured in trie.
        
        \begin{algorithm}
            \caption{KOM(K)}    \label{alg:KOM-PPM}
            \begin{algorithmic}[1]
                \State \textbf{Parameters:} Context Length, $k \in I^{+}_{0}$
                \State \textbf{Initialization:} Current Window, $W \gets \phi$
                \For{ n = 1, 2, \ldots}
                    \State Predict next symbol $d^{(k)}[n]$ based on (\ref{eqn:expertPrediction})
                    \State Observe actual symbol $s[n]$ and incur loss $l^{(k)}[n]$
                    \State $W \gets \left[W \: s[n]\right]$ (Append $s[n]$ to $W$)
                    \If{ $length(W) > k$ }
                        \State $W \gets W - w[1]$ (Remove first entry from $W$)
                    \EndIf
                    \For{ i = 1, \ldots, K }
                        \If{ context $W[i:k]$ is available in trie }
                            \State Increment context count by 1
                        \Else
                            \State Insert new context to trie and set value to 1
                        \EndIf
                    \EndFor
                \EndFor
            \end{algorithmic}
        \end{algorithm}
        
        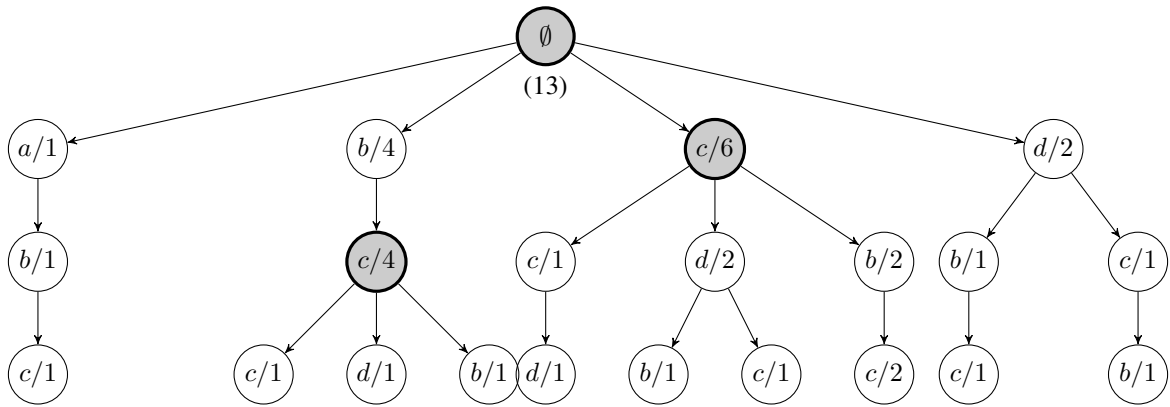
\begin{figure*}[!hbtp]
        \centering 
        \begin{tikzpicture}[ ->, 
                             >=stealth',
                             level/.style = { 
                                                 sibling distance = 4.5cm/#1,
                                                 level distance = 1.5cm
                                            }
                           ] 
            \node [selNode] [label=below:(13)] {$\emptyset$}
                child{
                    node [blankNode] {$a/1$}
                    child{
                        node [blankNode] {$b/1$}
                        child{
                            node [blankNode] {$c/1$}
                        }
                    }
                }
                child{
                    node [blankNode] {$b/4$}
                    child{
                        node [selNode] {$c/4$}
                        child{
                            node [blankNode] {$c/1$}
                        }
                        child{
                            node [blankNode] {$d/1$}
                        }
                        child{
                            node [blankNode] {$b/1$}
                        }
                    }
                }
                child{
                    node [selNode] {$c/6$}
                    child{
                        node [blankNode] {$c/1$}
                        child{
                            node [blankNode] {$d/1$}
                        }
                    }
                    child{
                        node [blankNode] {$d/2$}
                        child{
                            node [blankNode] {$b/1$}
                        }
                        child{
                            node [blankNode] {$c/1$}
                        }
                    }
                    child{
                        node [blankNode] {$b/2$}
                        child{
                            node [blankNode] {$c/2$}
                        }
                    }                    
                }
                child{ 
                    node [blankNode] {$d/2$}
                    child{
                        node [blankNode] {$b/1$}
                        child{
                            node [blankNode] {$c/1$}
                        }
                    }
                    child{
                        node [blankNode] {$c/1$}
                        child{
                            node [blankNode] {$b/1$}
                        }
                    }                    
                };
        \end{tikzpicture}
        \caption{Example tree constructed based on Algorithm \ref{alg:KOM-PPM} for sequence 
            $s = a, b, c, c, d, b, c, d, c, b, c, b, c$ with context length $2$.} \label{fig:exampleTree}
        \end{figure*}
        
        \textbf{Example:}\\
            Assume a sequence $s = a, b, c, c, d, b, c, d, c, b, c, b, c$. We consider a 
            tree with depth $3$; i.e., it can consider a context length of up to 2. Tree 
            constructed based on Algorithm \ref{alg:KOM-PPM} is given in Fig \ref{fig:exampleTree}. 
            At each node, the letter inside the node denotes the symbol stored at that 
            node and the numeric denotes the frequency of occurrence of that  context. At 
            root node, i.e., the node with zero context length, the frequency will be the 
            sequence length - 13 in this example. At the end of this sequence, we have a 
            context $\{b,c\}$. Applying PPM, we can get the symbol prediction probabilities, 
            $\mathbb{P}(\cdot|\cdot)$, as $\mathbb{P}(a|bc) = 1/312$, $\mathbb{P}(b|bc) = 108/312$, 
            $\mathbb{P}(c|bc) = 97/312$ and $\mathbb{P}(d|bc) = 106/312$. Using \ref{eqn:expertPrediction},
            we get $\hat{d}^{(2)}[14] = b$.

        A pool of experts is created as explained above in Algorithm \ref{alg:KOM-PPM}
        with values of $k = 0, \ldots, K-1$ forming the set $\mathcal{K}$. Rather than maintaining
        different tries for each expert, all experts can co-exist in the largest depth trie -
        the trie with context length equal to $K-1$. This helps to keep the memory requirement low
        and also satisfies our assumption of having a common decision space for all the experts.
        
                
        \begin{theorem} \label{thm:alz_conv}
            Algorithm \ref{alg:KOM-PPM} attains Finite State(FS) convergence at the rate of 
            $O\left(\sqrt{\frac{1}{n}}\right)$.
        \end{theorem} 
            \begin{proof}
                A predictor can be defined by the pair $(f,g)$, where $f$ is the next symbol
                prediction function and $g(\cdot)$ is the next-state function. Let $\pi(g;s_1^n)$
                be the minimum fraction of prediction errors, $s_1^n$ be sequence. Also let $S$ be the states
                in the predictor. Then finite state predictability is defined as \cite{Feder1992},
                \begin{align}
                    \pi(x) = \underset{\chi \rightarrow \infty}{\lim}
                                \underset{n \rightarrow \infty}{\lim \sup}
                                    \underset{g \in G_{\chi}}{\min} \pi(g;s_1^n)
                \end{align}
                This is the minimum fraction of error a predictor makes over the set of 
                available \textit{next-state} functions $G_{\chi}$, when both the number of states $\chi$ and the
                sequence length $n$ tends to infinity. Let $\hat{\pi}(g;s_1^n)$ be the expected
                fraction of errors over the randomization in $f(\cdot)$. Then, by Theorem 1 of
                \cite{Feder1992}, 
                \begin{align}
                    \hat{\pi}(g;s_1^n) \leq \pi(g; s_1^n) + 
                                            \frac{\chi}{n} \sqrt{\frac{n}{\chi}+1} + \frac{1}{2n}
                \end{align}
                Thus, $\hat{\pi}(g;x_1^n)$ approaches $\pi(g;x_1^n)$ atleast as fast as
                $O\left(\frac{\chi}{n} \sqrt{\frac{n}{\chi}}\right) = O\left(\sqrt{\frac{\chi}{n}} \right)$.
                
                In the case of Algorithm \ref{alg:KOM-PPM}, the number of states is evolving and an upper
                bound on the number of states is $|\mathcal{S}^{k+1}|$ as $n \rightarrow
                \infty$. Substituting this upper bound, we can get $O(\sqrt{\frac{\chi}{n}}) =
                O(\sqrt{\frac{|\mathcal{S}^{k+1}|}{n}}) = O(\sqrt{\frac{1}{n}})$. Thus, we
                can conclude that Algorithm \ref{alg:KOM-PPM} converges to FS predictability at a rate of
                $O(\frac{1}{\sqrt{n}})$ under a stationary environment. This result assumes that the 
                optimal order Markov Model exists in the frequency trie.
            \end{proof}
            
        This result is consistent with the result derived in \cite{Cover1977a}, where the authors show
        the results of Bayes predictors for which the expected proportion of errors of a Bayes
        predictor differ from the observed $k^{th}$ order Markov structure by $O(n^{-1/2})$.
        

    \section{Adaptive Prediction}    \label{sec:ada_prediction}
        In this section, we describes the algorithm for combining predictions from multiple experts.
        As there is no one single best expert for the whole time of prediction, the final predictor
        is required to adaptively combine the expert predictions based on observed performance. We
        propose a modified version of HEDGE algorithm (See \cite{Freund1997}) with a forgetting
        (discounting) factor to deal with this problem of shifting experts. 
        \begin{algorithm}[!h]
            \caption{Discounted HEDGE with PPM}    \label{alg:discountedHEDGE}
            \begin{algorithmic}[1]
                \State \textbf{Parameters: } $ \beta \in (0,1]$, $\gamma \in (0,1]$ $K \in I^+$ 
                \State \textbf{Initialization: } Set $w_k(1) = W > 0 \:\forall\: k$.
                \For{ $n = 1,2,\ldots$ }
                    \For{ $k = 1,\ldots,K$ }
                        \State $p^{(k)}[n] \gets \frac{w_k[n]^{\gamma}}{\sum \limits_{j=1}^{K}w_j[n]^{\gamma}}$    \label{eqn:normWeighCalc}
                    \EndFor
                    \State Get symbol distribution from each expert $p_{k}[n]$
                    \State Calculate $\hat{\textbf{P}} = \textbf{p}[n] \cdot 
                            [ p_{1}[n] \ldots p_{K}[n] ] $
                    \State Predicted symbol $\hat{s}[n] = \underset{s}{\arg \max} \: \hat{\textbf{P}}$
                    \State Observe actual symbol $s[n]$
                    \State Calculate loss $l^{k}[n] \: \forall \: k \in \{1,\ldots,K\}$
                    \State Update weights as $w_{k}[n+1] \gets w_k[n]^{\gamma} \cdot \beta^{l^{(k)}[n]}$    \label{eqn:weightUpdate}
                \EndFor
            \end{algorithmic}
        \end{algorithm}
        
    
        By maintaining a set of weights over the experts and
        doing multiplicative weight updation based on the observed loss, HEDGE is able to give
        more weight to best performing experts and almost zero weight to non-performing
        experts. The weight updation requires a factor $\beta \in (0,1]$
        to be set prior hand. By setting this parameter appropriately, HEDGE can achieve a
         upper bound of $O(\sqrt{N \ln K})$ for regret. These weights can be normalized to get
        the required weight distribution over the pool of experts.
        
        To introduce an adaptive behaviour for combining the outputs from multiple predictors,
        Algorithm \ref{alg:discountedHEDGE} includes a discounting factor $\gamma$, and hence 
        give more importance
        to the experts which have been performing well in the recent past. We modify 
        weight update to include this \textit{forgetting effect}
        to the final predictor. The proposed algorithm is given in Algorithm \ref{alg:discountedHEDGE}.
        \textit{Discounted HEDGE with PPM} requests prediction from each expert at every time instant
        and combine those based on the maintained distribution over experts. After actual symbol
        is observed, individual losses are calculated and weights are updated based on the
        performance of each expert.
        
        Let $\tilde{P}_{W_{i:j}}$ be the escape proabability calculated by PPM for the symbols
        in context window from $i$ to $j$.
        Prediction from Algorithm \ref{alg:KOM-PPM} and from Algorithm \ref{alg:discountedHEDGE} can 
        be deduced to matrix operations as shown below.
        
        \begin{figure}  [!h] 
        \scalebox{0.78}{\parbox{\linewidth}{%
            \begin{align*} 
                \begin{pmatrix}
                    \hat{P}(s_1) \\
                    \hat{P}(s_2) \\
                    \vdots \\
                    \hat{P}(s_z) \\
                \end{pmatrix}
                &=
                \begin{pmatrix}
                    \mathbb{P}_{s_1|w_{n-1:n-K+1}} &\mathbb{P}_{s_1|w_{n-1:n-K+2}} &\cdots &\mathbb{P}_{s_1} \\
                    \mathbb{P}_{s_2|w_{n-1:n-K+1}} &\mathbb{P}_{s_2|w_{n-1:n-K+2}} &\cdots &\mathbb{P}_{s_2} \\
                    \vdots &\vdots &\ddots &\vdots \\
                    \mathbb{P}_{s_z|w_{n-1:n-K+1}} &\mathbb{P}_{s_z|w_{n-1:n-K+2}} &\cdots &\mathbb{P}_{s_z}
                \end{pmatrix} 
                \begin{pmatrix}
                    1 &0 &\cdots &0 \\
                    \tilde{P}_{W_{n-1:K-1}} &1 &\cdots &0 \\
                    \tilde{P}_{W_{n-1:K-1}}\tilde{P}_{W_{n-1:K-2}} &\tilde{P}_{W_{n-1}\cdots x_{K-2}} &\cdots &0 \\
                    \vdots &\vdots  &\ddots &\vdots \\
                    \cdots &\cdots  &\cdots &1
                \end{pmatrix}
                \begin{pmatrix}
                    p^{(K)} \\
                    p^{(K-1)} \\
                    \vdots \\
                    p^{(1)}
                \end{pmatrix} \label{eqn:mat_form}
            \end{align*}
        }}
        \end{figure}
        
    \section{Regret Upper Bound}    \label{sec:regret_analysis}
    
        Next we derive regret upper bound for Algorithm \ref{alg:discountedHEDGE}. Our algorithm 
        analysis follows the same method as in \cite{Freund1997} but with two key observations 
        from majorization theory.
        
        First step is to relate the probability distributions with two different discounting 
        exponents through majorization, as given below.        
        \begin{lemma} \label{lem:majorization}
            Let $p^{(k)}(\eta) = \frac{w_k^{\eta}}{\sum \limits_{j=1}^{K} w_j^{\eta}}$
            where $\eta \in (0,1]$,  we have $\mathbf{p}^{(k)}(\gamma^M) \prec \mathbf{p}^{(k)}(\gamma)$
            $\forall M \geq 1$. That is $\mathbf{p}(\gamma^M)$ is majorized by $\mathbf{p}(\gamma)$. 
        \end{lemma}
        \begin{proof}
            Observing that $\gamma \in (0,1]$ and $\gamma^M \leq \gamma$ for all $M \geq 1$, this
            is a direct consequence of 5.B.2.b in \cite{Marshall2011}
        \end{proof}  
            
        \begin{lemma}    \label{lem:theInequality}
            Let the experts are ordered as $p^{(1)}(\gamma) \leq p^{(2)}(\gamma) \leq \ldots \leq 
            p^{(K)}(\gamma)$. If the instantaneous losses of the experts in same ordering obeys
            $l^{(1)} \geq l^{(2)} \geq \ldots \geq l^{(K)}$, then $\sum \limits_{k=1}^{K} 
            p^{(k)}(\gamma^M) \cdot l^{(k)} \geq \sum \limits_{k=1}^{K} p^{(k)}(\gamma)
            \cdot l^{(k)}$
        \end{lemma}
        \begin{proof}
            When $p^{(1)}(\gamma) \leq p^{(2)}(\gamma) \leq \ldots \leq p^{(K)}(\gamma)$, we can 
            directly observe that $p^{(1)}(\gamma^M) \leq p^{(2)}(\gamma^M) \leq \ldots 
            \leq p^{(K)}(\gamma^M)$. By Lemma \ref{lem:majorization}, $\mathbf{p}(\gamma^M)
            \prec \mathbf{p}(\gamma)$. Hence we have,
            \begin{align}
                \sum \limits_{i=1}^{k} p^{(i)}(\gamma^M) 
                    &\geq \sum \limits_{i=1}^{k} p^{(i)}(\gamma), \quad \forall \: k = 1,\ldots, K-1 \nonumber 
                \qquad \text{and} \qquad
                \sum \limits_{i=1}^{K} p^{(i)}(\gamma^M) 
                    &= \sum \limits_{i=1}^{K} p^{(i)}(\gamma). \nonumber
            \end{align}
            Let $m_j$ be some non-negative numbers. Then consider the sequence on inequalities, 
            \begin{align*}
                m_1 \cdot p^{(1)}(\gamma^M) &\geq m_1 \cdot p^{(1)}(\gamma) \nonumber \\
                m_2 \cdot \left( p^{(1)}(\gamma^M) + p^{(2)}(\gamma^M) \right)
                    &\geq m_2 \cdot \left( p^{(1)}(\gamma) + p^{(2)}(\gamma) \right) \nonumber \\
                \vdots \nonumber \\
                m_{K-1} \left( \sum \limits_{k=1}^{K-1} p^{(k)}(\gamma^M) \right)
                    &\geq m_{K-1} \cdot \left( \sum \limits_{k=1}^{K-1} p^{(k)}(\gamma) \right)  \nonumber \\
                m_K \left( \sum \limits_{k=1}^{K} p^{(k)}(\gamma^M) \right)
                    &= m_K \cdot \left( \sum \limits_{k=1}^{K} p^{(k)}(\gamma) \right) \nonumber \\
                \intertext{Summing over all of them,}
                \sum \limits_{k=1}^{K} \left( \sum \limits_{j=k}^{K} m_j \right) \cdot p^{(k)}(\gamma^M)
                    &\geq \sum \limits_{k=1}^{K} \left( \sum \limits_{j=k}^{K} m_j \right) \cdot p^{(k)}(\gamma)
            \end{align*}
            Taking $\sum \limits_{j=k}^{K} m_j = l^{(k)}$ and noting that
                $\sum \limits_{j=k}^{K} m_j \geq \sum \limits_{j=k+1}^{K} m_j$, we have
                $ \sum \limits_{k=1}^{K} l^{(k)} \cdot p^{(k)}(\gamma^M)
                    \geq \sum \limits_{k=1}^{K} l^{(k)} \cdot p^{(k)}(\gamma) $.
            This completes the proof.
        \end{proof}      
        \begin{lemma}
            Loss incurred by Discounted HEDGE algorithm can be upper bounded by the loss
            of best expert in pool as
            \begin{align}
                L_N(\gamma) &\leq \frac{\ln(1/\beta)}{1-\beta} L_{k^*,N}(\gamma) +
                                                                  \frac{\ln K}{1-\beta}.
            \end{align}
        \end{lemma}  \label{lem:disc_hedge_conv}
            \begin{proof}
                Our proof is partly based on the analysis of $HEDGE(\beta)$ (See \cite{Freund1997}).
                But the method of discounting we have introduced to the HEDGE algorithm leads 
                to certain technical difficulties in the proof which are addressed using 
                Lemma \ref{lem:majorization} and Lemma \ref{lem:theInequality}. Due to space
                constraints, only the key steps of the proof is provided below. For complete version
                of the proof, refer supplementary material.
            
                From Equation \ref{eqn:discExpertLoss} and Step \ref{eqn:weightUpdate} of 
                Algorithm \ref{alg:discountedHEDGE}, we get
                $
                    w_k[N+1] 
                        = w_k[1]^{\gamma^N} \cdot \beta^{L_{k,N}(\gamma)}  
                $.
                Consider the sum of weights of all experts at time instant $n+1$.
                \begin{align}
                    \sum \limits_{k=1}^{K} w_{k}[n+1] 
                        &= \sum \limits_{k=1}^{K} (w_{k}[n])^{\gamma} \cdot \beta^{l^{(k)}[n]}.  \label{eqn:sumOfWeights}
                \end{align}
                Applying Bernoulli's Inequality, we get
                \begin{align}
                    \sum \limits_{k=1}^{K} w_{k}[n+1]
                        &\leq \sum \limits_{k=1}^{K} (w_{k}[n])^{\gamma} - 
                            (1-\beta) \sum \limits_{k=1}^{K} (w_{k}[n])^{\gamma} \cdot l^{(k)}[n] \nonumber
                \end{align}
                Noting that
                $p^{(k)}[n] = \frac{w_k[n]}{\sum \limits_{j=1}^{K} w_j[n]}$ and applying
                Lemma \ref{lem:theInequality}, we can write
                \begin{align}
                    \sum \limits_{k=1}^{K} w_{k}[n+1] 
                        &\leq \left( \sum \limits_{k=1}^{K} (w_k[n])^{\gamma} \right) \cdot
                               \left( 1 -(1-\beta) \textbf{p}[n] \cdot \textbf{l}[n] \right) \nonumber
                \end{align}
                Applying $1+x \leq \exp(x)$ and expanding the terms recursively,
                \begin{align}
                    \sum \limits_{k=1}^{K} w_{k}[N+1] 
                        &\leq \sum \limits_{k=1}^{K} (w_{k}[1])^{\gamma^N}
                            \prod \limits_{n=1}^{N-1} \left( 1 - (1-\beta) \cdot \left( \sum \limits_{k=1}^{K} 
                            \frac{(w_k[n])^{\gamma^{N-n+1}}}{\sum \limits_{j=1}^{K} (w_j[n])^{\gamma^{N-n+1}}} \cdot l^{(k)}[n]\right) \right) \nonumber \\
                        &\qquad \qquad \qquad \qquad \qquad \qquad \qquad \qquad \qquad \cdot \exp\left( -(1-\beta) \textbf{p}[N] \cdot \textbf{l}[N] \right) \label{eq:prodInequality}
                \end{align}
                Because of the discounting that has been introduced to the HEDGE algorithm, we get terms of the
                form $\frac{(w_k[n])^{\gamma^{N-n+1}}}{\sum \limits_{j=1}^{K} (w_j[n])^{\gamma^{N-n+1}}}$,
                which cannot be readily used to calculate the expected loss of final predictor, as done
                in \cite{Freund1997}. 
                Without loss of generality, when the experts are arranged in ascending
                order of their weights, if their instantaneous losses follow a descending pattern then 
                by Lemma \ref{lem:theInequality}, we can write
                \begin{align}
                    \sum \limits_{k=1}^{K} w_{k}[N+1] 
                        &\leq \left( \sum \limits_{k=1}^{K} w_k[N-1]^{\gamma^{2}} \right)
                            \exp\left( -\gamma (1-\beta) \textbf{p}[N-1] \cdot \textbf{l}[N-1] \right)
                            \exp\left( -(1-\beta) \textbf{p}[N] \cdot \textbf{l}[N] \right) \nonumber
                \end{align}
                Combining the product terms and simplifying, we get               
                \begin{align}
                    \sum \limits_{k=1}^{K} w_{k}[N+1]   
                        &= \left( \sum \limits_{k=1}^{K} w_k[1]^{\gamma^N} \right) 
                               \exp \left( -(1-\beta) \cdot L_{N}(\gamma) \right) \nonumber
                               \quad (\because \text{Eqn.\ref{eqn:predictorLoss}})\nonumber
                \end{align}
                Taking logarithm on both sides and rearranging,
                \begin{align}
                    L_N(\gamma) 
                        &\leq - \frac{
                            \ln \left( \sum \limits_{k=1}^{K} w_k[N+1] \right) - 
                                \ln \left( \sum \limits_{k=1}^{K} w_k[1]^{\gamma^N} \right)
                        }{1-\beta}    \label{eqn:LN_midway}
                \end{align}
                Let $k^{*} = \underset{k \in \mathcal{K}}{\arg \min} \:L_{k,N}(\gamma)$ be the set index
                for the best expert in collection $\mathcal{K}$. Then we have
                $
                   \sum \limits_{k=1}^{K} w_k[N+1] 
                       \geq L_{k^*,N}(\gamma) \ln(\beta) + \ln \left( w_{k^*}[1]^{\gamma^N} \right) \nonumber
                $.
                Applying this to (\ref{eqn:LN_midway}) and by setting 
                $w_k[1] = W \:\forall\: k \in \mathcal{K}$ with $ W > 0$, we get
                \begin{align}
                    L_N(\gamma)
                        &\leq \frac{\ln(1/\beta)}{1-\beta} L_{k^*,N}(\gamma) +
                                    \frac{\ln K}{1-\beta}        \label{eqn:lossBoundRaw}
                \end{align}
                This completes the proof.
            \end{proof}

            (\ref{eqn:lossBoundRaw}) is similar in structure to the the loss bound of HEDGE algorithm.
            Next, we prove that by optimally setting the value of $\beta$, we can bound this to 
            linear term.

        \begin{theorem}
            By optimally setting value of $\beta$, Net Loss of PPM-HEDGE algorithm can be bounded 
            by $O(\sqrt{N \ln K})$.
        \end{theorem} \label{thm:discHedgebound}
            \begin{proof}
                From Lemma 4 in \cite{Freund1997}, \textit{Suppose $0 \leq L \leq \tilde{L}$ and $0 \leq R \leq
                \tilde{R}$ and $\beta = g(\tilde{L}/\tilde{R})$ where $g(z) = 1 /(1 + \sqrt{2/z})$, then}
                \begin{align}
                    \frac{-L \ln(\beta) + R}{1-\beta} \leq L + \sqrt{2 \tilde{L} \tilde{R}} + R    \label{eqn:optiBetaLemma}
                \end{align}
                Taking $L \leq L_{k^*,N}(\gamma)$ and $R = \ln(K)$, (\ref{eqn:lossBoundRaw}) can be
                rewritten as,
                \begin{align*}
                    L_{N}(\gamma) = L_{k^*,N}(\gamma) + \sqrt{2 \tilde{L} \tilde{R}} + \ln(K)
                \end{align*}
                From (\ref{eqn:discExpertLoss}), we get
                $
                    L_{k,N} = \sum \limits_{n=1}^{N} \gamma^{N-n} l^{(k)}[n] 
                            \leq \sum \limits_{n=1}^{N} \gamma^{N-n}
                            = \sum \limits_{n=0}^{N-1} \gamma^{n}
                            = \frac{1-\gamma^N}{1-\gamma}
                $. 
                This is the maximum value that the loss can take and hence, 
                $
                    \tilde{L} = \frac{1-\gamma^N}{1-\gamma}
                $.
                As a limiting case, when $\gamma = 1.00$,
                $
                    \tilde{L} = \lim \limits_{\gamma \rightarrow 1} \frac{1-\gamma^N}{1-\gamma}
                              = N        
                $.
                For a particular prediction task, $R$ is fixed and hence, 
                $
                    \tilde{R} = \ln (K) 
                $.
                This will give the optimal value of $\beta$ as
                \begin{align}
                    \beta &= 
                        \begin{cases}
                            \frac{1}{1+\sqrt{2 \cdot \ln(K) \cdot \frac{1-\gamma}{1-\gamma^N}}} &; \gamma \in [0,1) \\
                            \frac{1}{1+\sqrt{2 \cdot \ln(K) / N}}    &; \gamma = 1
                    \end{cases}    \label{eqn:optimalBeta}
                \end{align}
                Applying above results in (\ref{eqn:lossBoundRaw}), we get
                \begin{align}
                    L_N(\gamma) &\leq
                        \begin{cases}
                            L_{k^*,N}(\gamma) + \sqrt{2 \cdot \ln(K) \cdot \frac{1-\gamma^N}{1-\gamma}} + \ln(K)
                                                                                         &;\gamma \in [0,1)\\
                            L_{k^*,N}(\gamma) + \sqrt{2 \cdot \ln(K) \cdot N} + \ln(K)   &; \gamma = 1
                        \end{cases}    \label{eqn:upperBound}
                \end{align}
                This completes the proof.
            \end{proof}

        \begin{corollary}
            By setting $\gamma = 1$, Discounted HEDGE algorithm becomes HEDGE with no forgetting
            factor.
        \end{corollary}   
    
    
    \section{Experimental Results}    \label{sec:experimental_results}
    
        To evaluate usefulness of the proposed method, this section provides comparison of
         the proposed algorithm with six other widely used and state-of-the-art algorithms -
        LZ78 (See \cite{Ziv1978}),
        Dependency Graphs (See \cite{Padmanabhan1996}),
        LeZi Update (See \cite{Bhattacharya2002}), 
        Active LeZi (See \cite{Gopalratnam2003}), 
        Error Weighted PPM (See \cite{Pulliyakode2014}) (referred as \textit{ewPPM} in results) and 
        Adaptive MPP (See \cite{Vladimir2017}). We show the results 
        on four different real life datasets - Reality Mining Dataset (RM) (See \cite{Eagle2006}),
        Building Activity (BA) (See \cite{casas_wsu}), Cognitive Assessment (CA) (See \cite{casas_wsu})
        and a proprietary dataset of Modulation Scheme prediction (MCS) from a LTE cellular
        network comprising of 19 cells, 3 sector layout containing MCS values for 570 users
        corresponding to the rate feedback from the cellphones. In the results, \textit{dHedgePPM} refers
        to the proposed method. For ewPPM, Adaptive MPP, and dHedge, we used a model built according
        to Algorithm \ref{alg:KOM-PPM} with $K=4$ and hence, have $5$ experts to predict with.
        
        \subsection{Loss Model for Experiments}
            In the derivation of the bound, the loss function is defined as $\mathcal{L}(x,y) \in [0,1]$,
            where $x$ is the observed symbol and $y$ is the predicted symbol.
            This enables us to use any loss function satisfying the above criteria and derive
            different bounds based on the applications. For the purpose of validating the claims
            presented above, and in order to prove the utility of the algorithm, we are using the
            following discrete loss function for the experiments mentioned below.
            \begin{align}    \label{eqn:lossFunction}
                \mathcal{L}(x,y) =  \llbracket x \neq y \rrbracket
            \end{align}
            where $\llbracket \cdot \rrbracket$ is the indicator function. Hence $\mathcal{L}(x,y)
            \in \{0,1\}$. 
            
            Even though we made some strict assumptions about the ordering of the weights of the
            experts and their corresponding instantaneous 
            losses, our experiments show that the results hold even when these conditions are not 
            always met. This shows the possibility of having a wider set of scenarios where the 
            proposed analysis can hold.

        \subsection{Performance Comparison}
            In order to compare the accuracy evolution with time, we considered a prediction
            task on first $5000$ symbols of each of the sequences in all the datasets except
            for MCS dataset. In MCS dataset, there are $210$ sequences of $992$ symbols each.
            For RM, BA, and CA datasets, we have $95$, $9$ and $350$ sequences respectively. For the
            algorithms which require a context length to work with, we set it to $4$ after
            making a few empirical observations. We provide the result of average accuracy over all the
            sequences in the dataset with respect to time.
            Figure \ref{fig:results_accuracy} shows
            the results.
            
            \begin{figure*}[t]
                \begin{subfigure}[h]{0.49\textwidth}
                    \includegraphics[width=\linewidth]{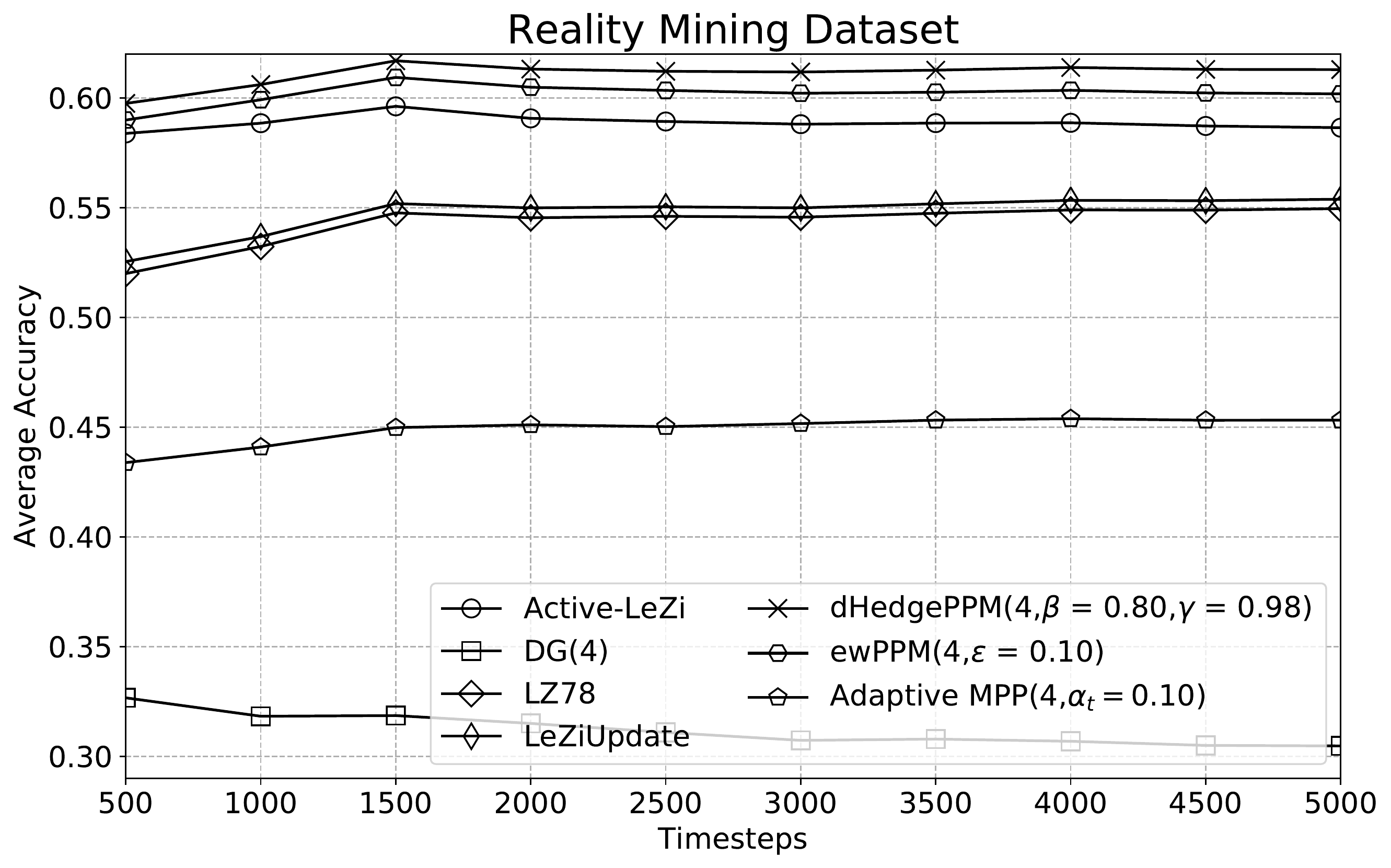}
                    \caption{Reality Mining Dataset}
                    \label{fig:acc_rm}
                \end{subfigure}%
                \begin{subfigure}[h]{0.49\textwidth}
                    \includegraphics[width=\linewidth]{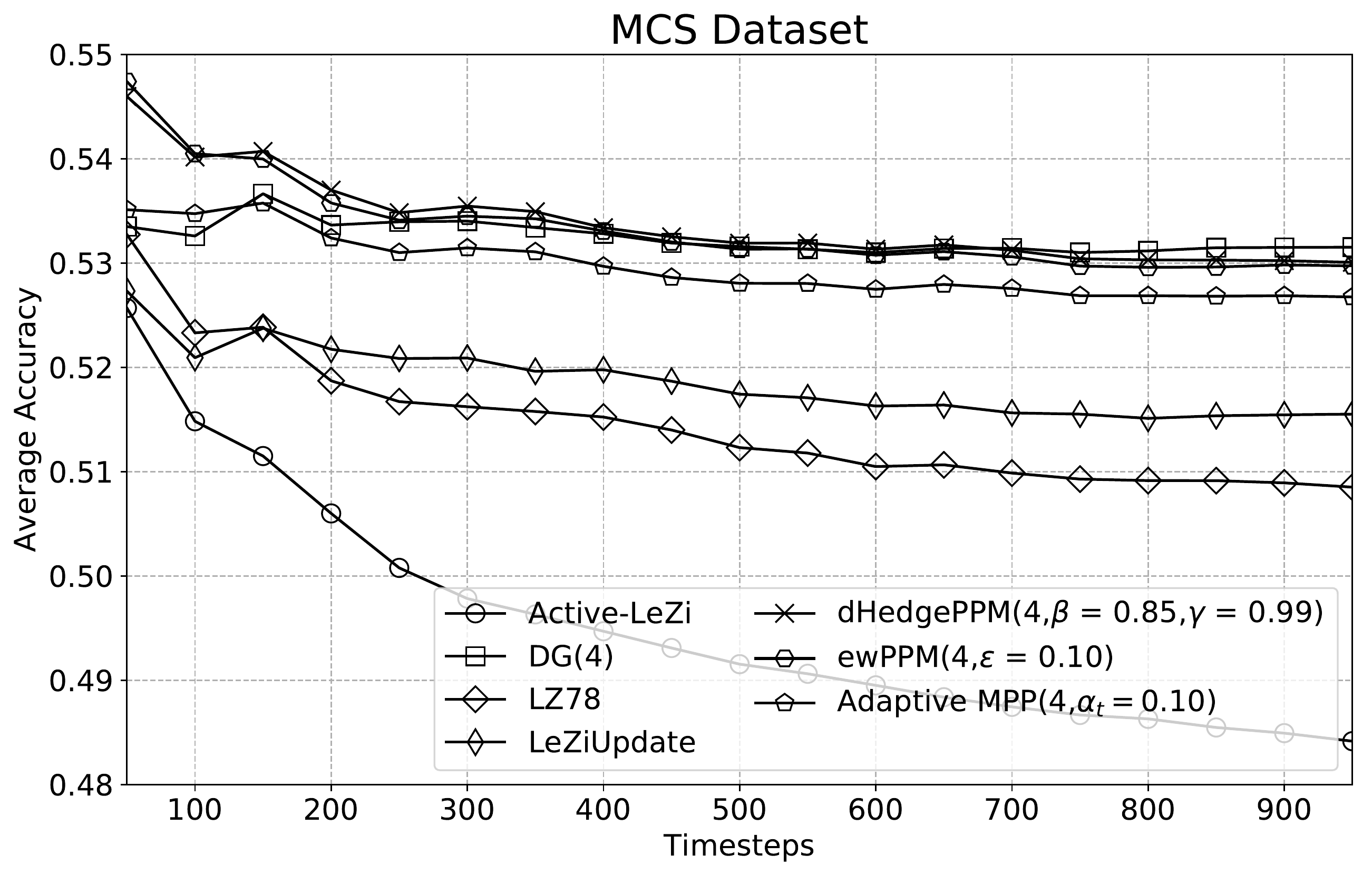}
                    \caption{MCS Prediction Dataset}
                    \label{fig:acc_mcs}
                \end{subfigure}\\
                \begin{subfigure}[h]{0.49\textwidth}
                    \includegraphics[width=\linewidth]{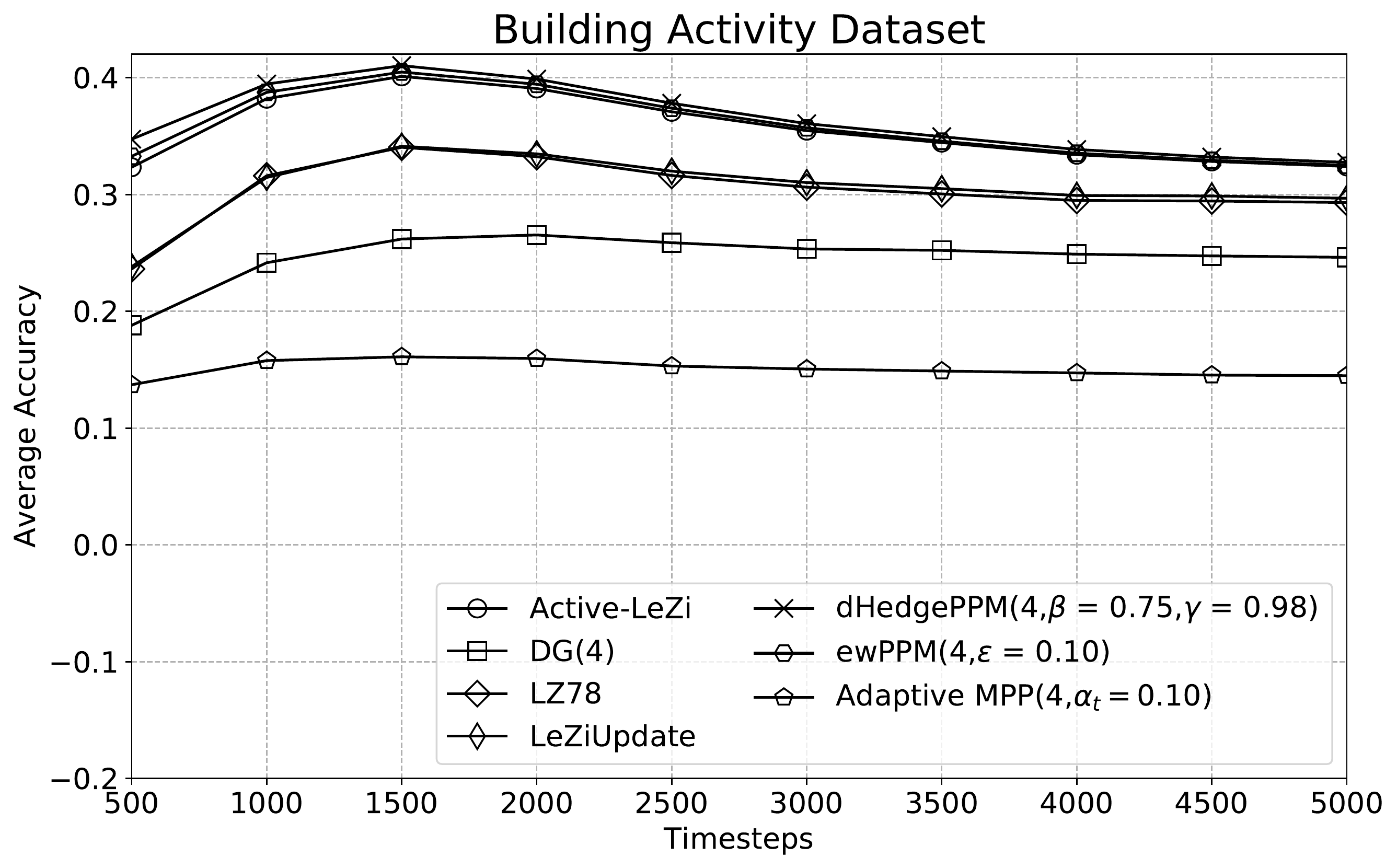}
                    \caption{Building Activity Dataset}
                    \label{fig:acc_ba}
                \end{subfigure}%
                \begin{subfigure}[h]{0.49\textwidth}
                    \includegraphics[width=\linewidth]{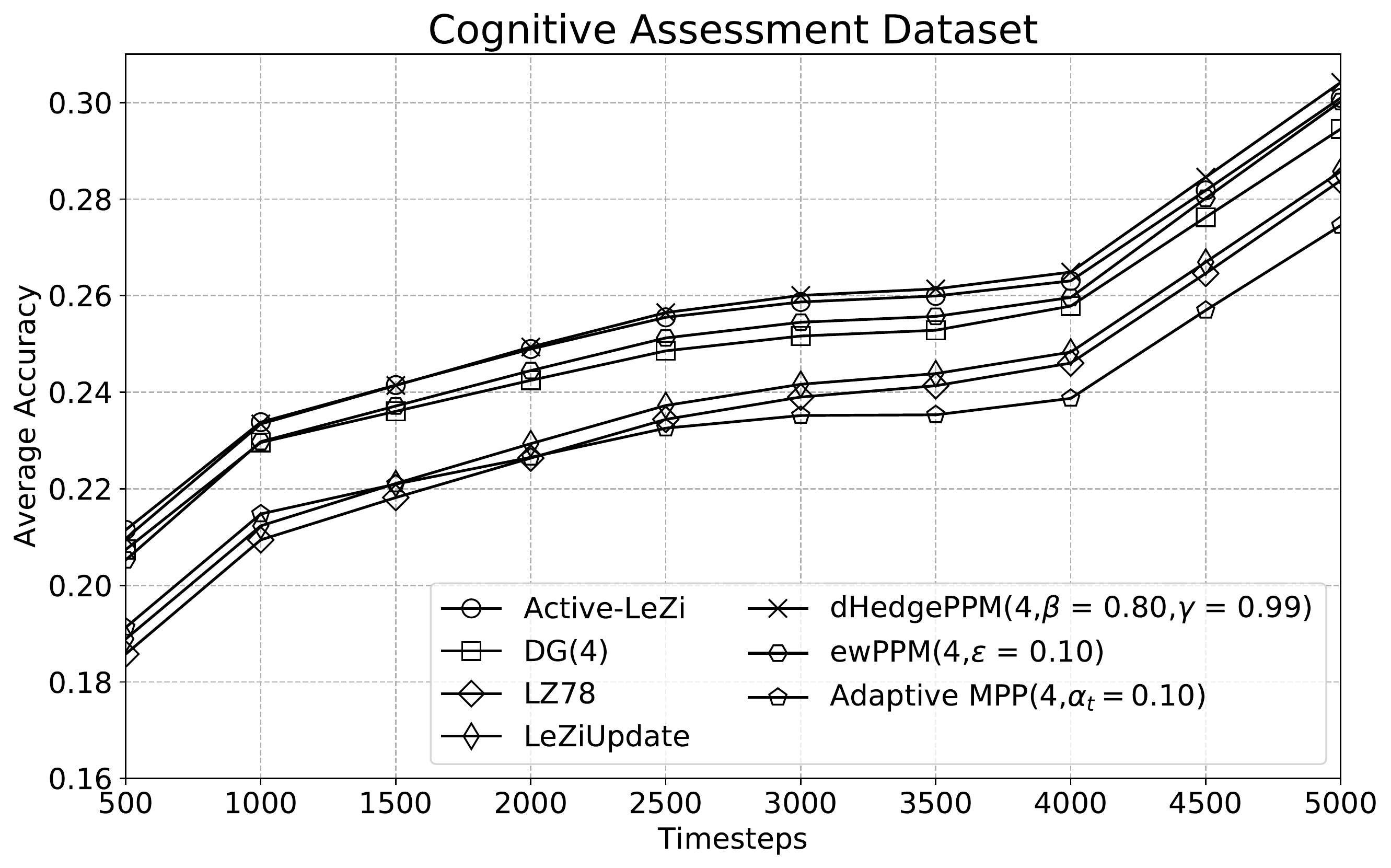}
                    \caption{Cognitive Assessment Dataset}
                    \label{fig:acc_ca}
                \end{subfigure}%
                \caption{Evolution of accuracy over time.} \label{fig:results_accuracy}
            \end{figure*}
            
            In each dataset, although different algorithms are performing better, the proposed
            method is able to match the performance of the best algorithm in all the datasets.
            Even though Active-LeZi is able to perform satisfactorily in RM, BA and CA, it suffers a severe
            hit in MCS dataset. In MCS dataset, DG eventually performs better. But in all cases,
            the proposed method can be seen performing nearly as good as the best performer. Another
            interesting observation is the performance of ewPPM; even though its performance
            is comparable to that of proposed method, it lacks the theoretical guarantees enjoyed
            by the proposed method. 

            Another interesting observation in the case of BA and MCS prediction dataset is the
            decline in prediction accuracy after some time. This is due to the increase in
            entropy of the underlying model as the time progresses. Particularly, the 
            performance of Active-LeZi in MCS dataset is interesting - the accuracy reduces sharply
            as the time progresses. Even though we know that the entropy in the system increases,
            it will be worthwhile to examine this behavior to get better insights on how Active-LeZi
            works.

        \subsection{Comparison of execution time}
            All the algorithms are implemented on the programming language Julia (See \cite{Julia2017}) 
            and is ran on a Linux machine with Intel Core i5@2.90Ghz CPU and 16GB RAM. The time taken for the
            running each of the algorithms over the entire dataset is listed in 
            Table \ref{tab:complexity_time}.
            \begin{table}[!h]
                \centering
                \begin{tabular}{l r r r r}
                    \hline
                    \hline
                    Algorithm               & RM       & CA       & BA      &  MCS      \\
                    \hline
                    Active-LeZi             &  $158.1$ &  $170.3$ &   $8.1$ &  $12.6$   \\
                    Dependency Graph        &    $8.9$ &   $21.4$ &   $1.1$ &   $1.6$   \\
                    LZ78                    & $1422.6$ & $4950.4$ & $193.6$ & $123.1$   \\
                    LZ Update               &   $43.6$ &   $68.9$ &   $3.4$ &   $5.3$   \\
                    ewPPM                   &  $157.1$ &  $254.7$ &  $16.4$ &  $10.2$   \\
                    Adaptive MPP            &  $154.1$ &  $558.8$ &  $15.5$ &   $8.7$   \\
                    Proposed Method         &  $156.1$ &  $256.0$ &  $16.7$ &   $8.5$   \\
                    \hline
                    \hline
                \end{tabular}
                \caption{Execution Time for algorithms (in seconds)}
                \label{tab:complexity_time}
            \end{table}

            We can see that DG and LZ-Update are fast algorithms, but their performances are 
            not consistent across different datasets. A notable observation is the runtime
            of LZ78; even though it is supposed to be low, the additional check during the
            trie building process contributes to a high observed runtime of LZ78.
        
        \subsection{Comparison of memory requirements}
            For comparing the memory requirements, we provide the number of symbol nodes
            constructed by each of the algorithms. Each symbol node will hold space for
            the actual symbol and a number representing its support over its child nodes.
            The comparison is provided in Table \ref{tab:complexity_time}.
            \begin{table}[!h]
                \centering
                \begin{tabular}{l r r r r}
                    \hline
                    \hline
                    Algorithm               & RM         & CA         & BA        &  MCS        \\
                    \hline
                    Active-LeZi             &  $4098544$ &  $6683543$ &  $226030$ &  $493660$   \\
                    Dependency Graph        &   $168201$ &   $489968$ &   $34803$ &   $13946$   \\
                    LZ78                    &   $117041$ &   $493746$ &   $16415$ &   $45593$   \\
                    LZ Update               &   $154875$ &   $592527$ &   $20853$ &   $57362$   \\
                    ewPPM                   &   $505316$ &  $3371416$ &  $112260$ &  $181178$   \\
                    Adaptive MPP            &   $505316$ &  $3371416$ &  $112260$ &  $181178$   \\
                    Proposed Method         &   $505316$ &  $3371416$ &  $112260$ &  $181178$   \\
                    \hline
                    \hline
                \end{tabular}
                \caption{Number of symbol nodes created}
                \label{tab:complexity_space}
            \end{table}
            As mentioned earlier, ewPPM, Adaptive MPP, and the proposed method use the same trie
            building strategy and hence, have the same memory requirements. We can observe that the
            proposed algorithm is not as efficient as LZ78 and DG in terms of memory requirement.
            But, the advantage is clearly visible in the prediction accuracy.
    
    \section{Concluding Remarks}
        We introduced an algorithm for adaptively combining multiple experts in a non-stationary
        environment and applied it to the task of sequence prediction. We also
        derived an upper bound on the regret of its loss. Numerical verification is performed with
        six other widely used sequence prediction/aggregating strategies
        to prove the utility of the proposed algorithm. The main advantage
        is that the proposed method can be used as a drop in replacement for conventional 
        PPM methods without
        changing other parts of the system, and yet offer an improvement in performance.
    
        Even though we applied our adaptive expert combining method over PPM for the task of discrete
        sequence prediction, the method proposed in this work can be applied to a wide variety of
        problems that require combining opinions from multiple experts when the best expert keeps on
        changing. One of the potential
        applications could be to consider different algorithms like Context Tree Weighting, Probability
        Suffix Tree along with PPM to create a pool of experts and then use the proposed method
        to predict based on the knowledge acquired by all experts.

        An interesting direction for further studies will be the derivation of tighter
        bounds as the experiments show that the bound proposed in this work can be improved. 
        Prior knowledge about the accuracy
        of the experts in the pool might be the key to a tighter bound.
        

    \bibliographystyle{plain}
    \bibliography{./../library.bib}
    
    \section*{A. Proof of Lemma 4}
    
        \begin{proof}
                Our proof is partly based on the analysis of $HEDGE(\beta)$ (See \cite{Freund1997}).
                But the method of discounting we have introduced to the HEDGE algorithm leads 
                to certain technical difficulties in the proof which are addressed using 
                majorization theory.            
                From Algorithm 2, we have
                \begin{align}
                    w_k[N+1] &= w_k[N]^{\gamma} \cdot \beta^{l^{(k)}[N]} \nonumber \\
                        &= (w_k[N-1] \cdot \beta^{l^{(k)}[N-1]})^{\gamma} \cdot \beta^{l^{(k)}[N]} \nonumber \\
                        &= w_k[1]^{\gamma^N} \cdot \beta^{L_{k,N}(\gamma)}  \qquad \qquad \text{(From (Eqn.1))}. \label{eqn:weightTraceBack} 
                \end{align}
                Now consider the sum of weights of all experts at time instant $n+1$.
                \begin{align}
                    \sum \limits_{k=1}^{K} w_{k}[n+1] 
                        &= \sum \limits_{k=1}^{K} (w_{k}[n])^{\gamma} \cdot \beta^{l^{(k)}[n]}.
                \end{align}
                By Bernoulli's Inequality, for $\beta \geq -1$ and $l^{(k)}[n] \in [0,1]$,
                \begin{align*}
                    \beta^{l^{(k)}[n]} \leq 1 - (1-\beta)l^{(k)}[n].
                \end{align*}
                For \textit{Discounted HEDGE}, we have $\beta \in (0,1]$. Applying this in (\ref{eqn:sumOfWeights}), we have
                \begin{align}
                    \sum \limits_{k=1}^{K} w_{k}[n+1]
                        &\leq \sum \limits_{k=1}^{K} (w_{k}[n])^{\gamma} \cdot \left( 1-(1-\beta)l^{(k)}[n]\right) \nonumber \\
                        &= \sum \limits_{k=1}^{K} (w_{k}[n])^{\gamma} - 
                            (1-\beta) \sum \limits_{k=1}^{K} (w_{k}[n])^{\gamma} \cdot l^{(k)}[n] \nonumber
                \end{align}
                From Algorithm 2, noting that
                $p^{(k)}[n] = \frac{(w_k[n])^\gamma}{\sum \limits_{j=1}^{K} (w_j[n])^\gamma}$, we can write
                \begin{align}
                    \sum \limits_{k=1}^{K} w_{k}[n+1] 
                        &\leq \sum \limits_{k=1}^{K} (w_{k}[n])^{\gamma} -
                            (1-\beta) \left( \sum \limits_{j=1}^{K} (w_{k}[n])^{\gamma} \right)
                                \left( \sum \limits_{k=1}^{K} p^{(k)}[n] \cdot l^{(k)}[n]\right) \nonumber \\
                        &= \left( \sum \limits_{k=1}^{K} (w_k[n])^{\gamma} \right)
                               \left( 1 -(1-\beta)
                                   \left( \sum \limits_{k=1}^{K} p^{(k)}[n] \cdot l^{(k)}[n] \right) 
                               \right) \nonumber \\
                        &= \left( \sum \limits_{k=1}^{K} (w_k[n])^{\gamma} \right) \cdot
                               \left( 1 -(1-\beta) \textbf{p}[n] \cdot \textbf{l}[n] \right) \nonumber
                \end{align}
                Applying $1+x \leq \exp(x)$, we can write
                \begin{align}
                    \sum \limits_{k=1}^{K} w_{k}[N+1] 
                        &\leq \left( \sum \limits_{k=1}^{K} (w_k[N])^{\gamma} \right) 
                               \exp\left( -(1-\beta) \textbf{p}[N] \cdot \textbf{l}[N] \right) \nonumber \\
                        &= \left( \sum \limits_{k=1}^{K}
                               \left( w_k[N-1]^{\gamma^{2}} \cdot \beta^{\gamma l^{(k)}[N-1]} \right)
                            \right) 
                            \exp\left( -(1-\beta) \textbf{p}[N] \cdot \textbf{l}[N] \right) \nonumber \\
                        &\leq \sum \limits_{k=1}^{K} (w_{k}[N-1])^{\gamma^2}
                            \left( 1 - (1-\beta) \cdot \left( \sum \limits_{k=1}^{K} 
                        \frac{(w_k[N-1])^{\gamma^{2}}}{\sum \limits_{j=1}^{K} (w_j[N-1])^{\gamma^{2}}} \cdot l^{(k)}[N-1]\right) \right) \nonumber \\
                        &\qquad \qquad \qquad \qquad \qquad \qquad \qquad \qquad \qquad \qquad \cdot \exp\left( -(1-\beta) \textbf{p}[N] \cdot \textbf{l}[N] \right) \nonumber \\
                        &\leq \sum \limits_{k=1}^{K} (w_{k}[1])^{\gamma^N}
                            \prod \limits_{n=1}^{N-1} \left( 1 - (1-\beta) \cdot \left( \sum \limits_{k=1}^{K} 
                            \frac{(w_k[n])^{\gamma^{N-n+1}}}{\sum \limits_{j=1}^{K} (w_j[n])^{\gamma^{N-n+1}}} \cdot l^{(k)}[n]\right) \right) \nonumber \\
                        &\qquad \qquad \qquad \qquad \qquad \qquad \qquad \qquad \qquad \cdot \exp\left( -(1-\beta) \textbf{p}[N] \cdot \textbf{l}[N] \right)
                \end{align}
               
                Because of the discounting that has been introduced to the HEDGE algorithm, we get terms of the
                form $\frac{(w_k[n])^{\gamma^{N-n+1}}}{\sum \limits_{j=1}^{K} (w_j[n])^{\gamma^{N-n+1}}}$,
                which cannot be readily used to calculate the expected loss of final predictor, as done
                in \cite{Freund1997}. However, if 
                \begin{align}    
                    \sum \limits_{k=1}^{K} 
                        \frac{(w_k[n])^{\gamma^{N-n}}}{\sum \limits_{j=1}^{K} (w_j[n])^{\gamma^{N-n}}} \cdot l^{(k)}[n]
                        \geq
                        \sum \limits_{k=1}^{K} 
                            \frac{(w_k[n])^{\gamma}}{\sum \limits_{j=1}^{K} (w_j[n])^{\gamma}} \cdot l^{(k)}[n] \label{eq:upperBoundOnP}
                \end{align}
                we can upper bound the LHS in (\ref{eq:prodInequality}) with expression involving instantaneous
                probability terms. Hence if (\ref{eq:upperBoundOnP}) holds, we will be able proceed in a manner similar
                to the analysis in \cite{Freund1997}. We show that such an inequality does hold provided
                certain conditions are met. Without loss of generality, when the experts are arranged in ascending
                order of their weights, if their instantaneous losses follow a descending pattern then 
                the above inequality holds. Refer to Appendix 2 for the  proof.                
                Now using (\ref{eq:upperBoundOnP}), we can rewrite (\ref{eq:prodInequality}) as
                \begin{align}
                    \sum \limits_{k=1}^{K} w_{k}[N+1] 
                        &\leq \left( \sum \limits_{k=1}^{K} w_k[N-1]^{\gamma^{2}} \right)
                            \exp\left( -\gamma (1-\beta) \textbf{p}[N-1] \cdot \textbf{l}[N-1] \right)
                            \exp\left( -(1-\beta) \textbf{p}[N] \cdot \textbf{l}[N] \right) \nonumber
                \end{align}
               
                Combining the product terms and simplifying, we get
               
                \begin{align}
                    \sum \limits_{k=1}^{K} w_{k}[N+1]   
                        &\leq \left( \sum \limits_{k=1}^{K} w_k[N-1]^{\gamma^{2}} \right)
                            \prod \limits_{n=N-1}^{N}
                                    \exp (-(1-\beta) \cdot \gamma^{N-n} \textbf{p}[n] \textbf{l}[n] ) \nonumber \\
                        &= \left( \sum \limits_{k=1}^{K} w_k[N-1]^{\gamma^{2}} \right)
                            \left( 
                                   -(1-\beta) \sum \limits_{n=N-1}^{N} \gamma^{N-n} \textbf{p}[n] \textbf{l}[n] 
                               \right)  \nonumber \\
                        &\leq \left( \sum \limits_{k=1}^{K} w_k[1]^{\gamma^N} \right)
                            \exp \left( 
                                   -(1-\beta) \sum \limits_{n=1}^{N} \gamma^{N-n} \textbf{p}[n] \textbf{l}[n]
                               \right) \nonumber \\
                        &= \left( \sum \limits_{k=1}^{K} w_k[1]^{\gamma^N} \right) 
                               \exp \left( -(1-\beta) \cdot L_{N}(\gamma) \right) \nonumber
                               \quad (\because \text{Eqn.1 from main paper})\nonumber
                \end{align}
                
                Taking logarithm on both sides,
                \begin{align}
                    \ln \left( \sum_{k=1}^{K} w_k[N+1] \right)
                        &\leq -(1-\beta) \cdot L_N(\gamma) + \ln \left( \sum \limits_{k=1}^{K} w_k[1]^{\gamma^N} \right) \nonumber
                \end{align}
                Rearranging,
                \begin{align}
                    L_N(\gamma) 
                        &\leq - \frac{
                            \ln \left( \sum \limits_{k=1}^{K} w_k[N+1] \right) - 
                                \ln \left( \sum \limits_{k=1}^{K} w_k[1]^{\gamma^N} \right)
                        }{1-\beta}
                \end{align}
                Let $k^{*} = \underset{k \in \mathcal{K}}{\arg \min} \:L_{k,N}(\gamma)$ be the set index
                for the best expert in collection $\mathcal{K}$. Then we have
                \begin{align}
                   \sum \limits_{k=1}^{K} w_k[N+1] 
                       &= \sum \limits_{k=1}^{K} w_k[1]^{\gamma^N} \cdot \beta^{L_{k,N}(\gamma)} \qquad(\because \text{From (\ref{eqn:weightTraceBack})}) \nonumber  \\
                    &\geq \beta^{L_{k^{*},N}(\gamma)} \cdot w_{k^{*}}[1]^{\gamma^N} \nonumber \\
                    \intertext{Taking logarithm,}
                    \ln \left( \sum \limits_{k=1}^{N} w_k [N+1] \right)
                    &\geq L_{k^*,N}(\gamma) \ln(\beta) + \ln \left( w_{k^*}[1]^{\gamma^N} \right) \nonumber
                \end{align}
                Applying this to (\ref{eqn:LN_midway}), we have
                \begin{align}
                    L_N(\gamma) 
                        &\leq -\frac{
                            L_{k^*,N}(\gamma) \ln(\beta)}{1-\beta}
                             - \frac{ 
                                \ln \left( w_{k^*}[1]^{\gamma^N} \right) -
                                    \ln \left( \sum \limits_{k=1}^{K} w_k[1]^{\gamma^N} \right)
                           }{1-\beta} \nonumber \\
                        &= \frac{\ln(1/\beta)}{1-\beta} L_{k^*,N}(\gamma) - 
                                \frac{1}{1-\beta} \ln \left( \frac{ w_{k^*}[1]^{\gamma^N} } 
                                   {\sum \limits_{k=1}^{K} w_k[1]^{\gamma^N}} \right) \nonumber
                \end{align}
            
                By setting $w_k[1] = W \:\forall\: k \in \mathcal{K}$ with $ W > 0$, we can observe that
                \begin{align}
                    L_N(\gamma)
                        &\leq \frac{\ln(1/\beta)}{1-\beta} L_{k^*,N}(\gamma) +
                                    \frac{\ln K}{1-\beta}        
                \end{align}
            \end{proof} 

    \section*{B. Validity of Proposed Bound}
        \begin{figure*}[t]
                \begin{subfigure}[h]{0.33\textwidth}
                    \includegraphics[width=\linewidth]{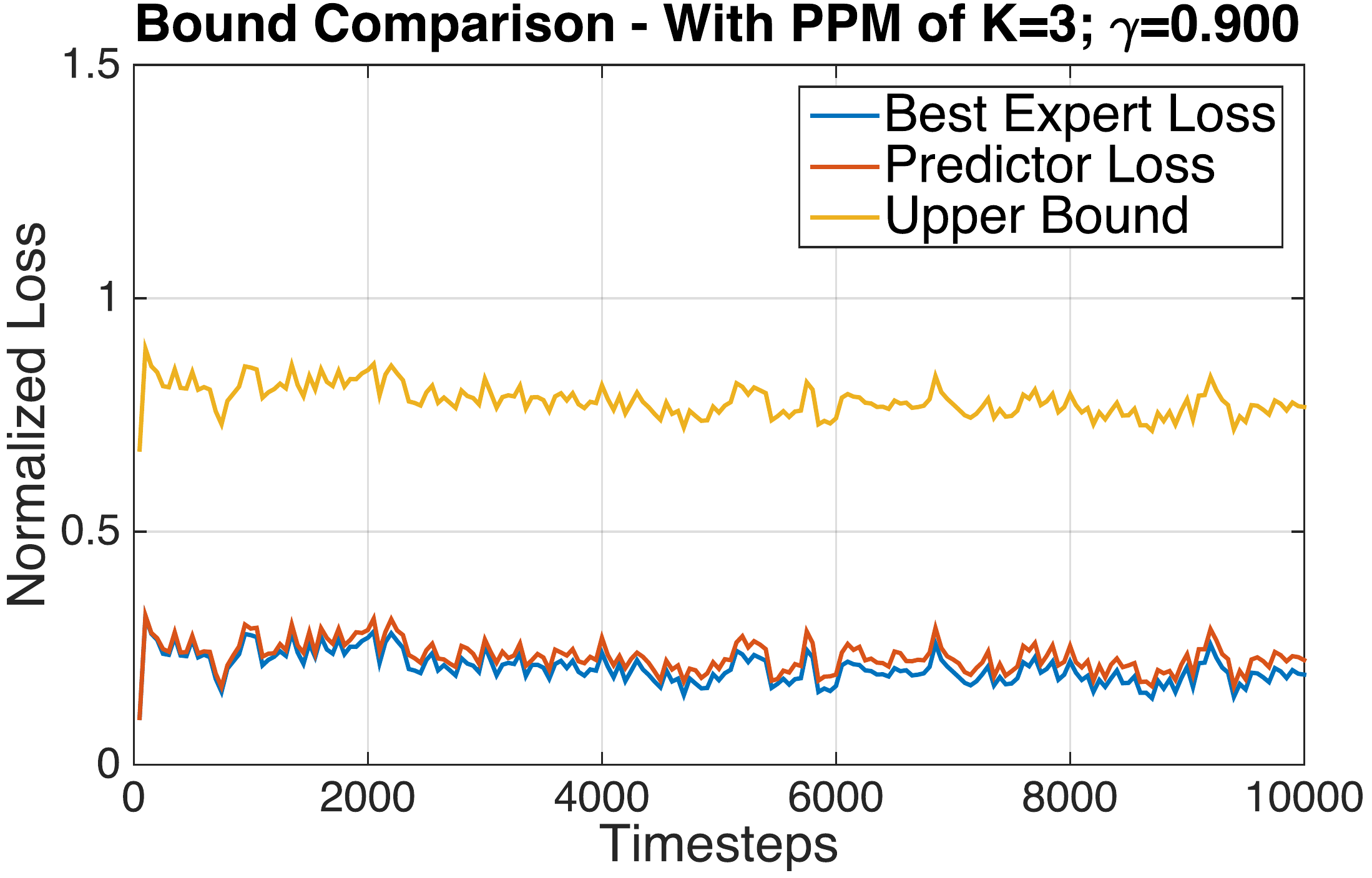}
                    \caption{No.of Experts = 3}
                    \label{fig:dscSim_g090_k03}
                \end{subfigure}%
                \begin{subfigure}[h]{0.33\textwidth}
                    \includegraphics[width=\linewidth]{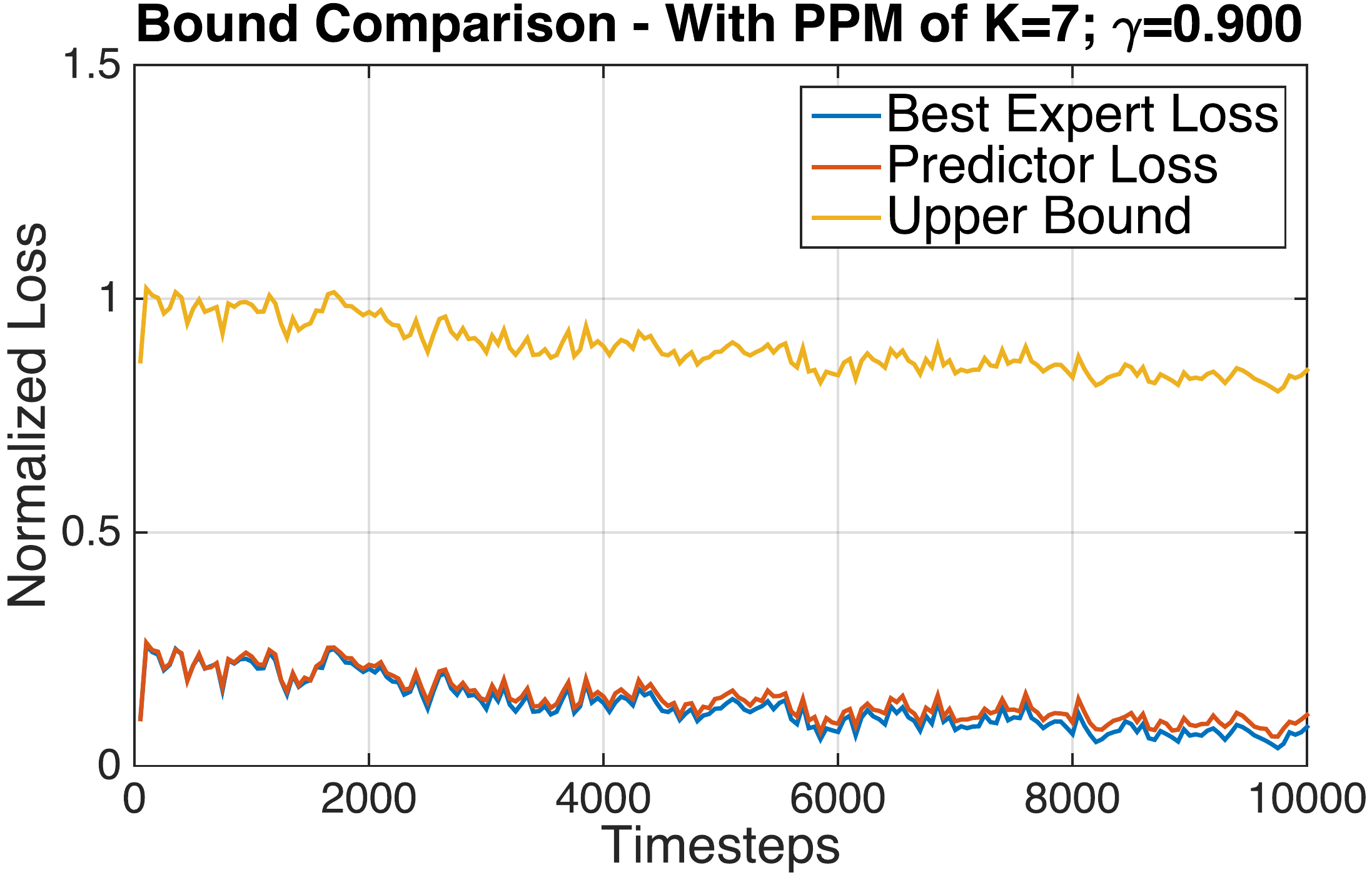}
                    \caption{No.of Experts = 7}
                    \label{fig:dscSim_g090_k07}
                \end{subfigure}%
                \begin{subfigure}[h]{0.33\textwidth}
                    \includegraphics[width=\linewidth]{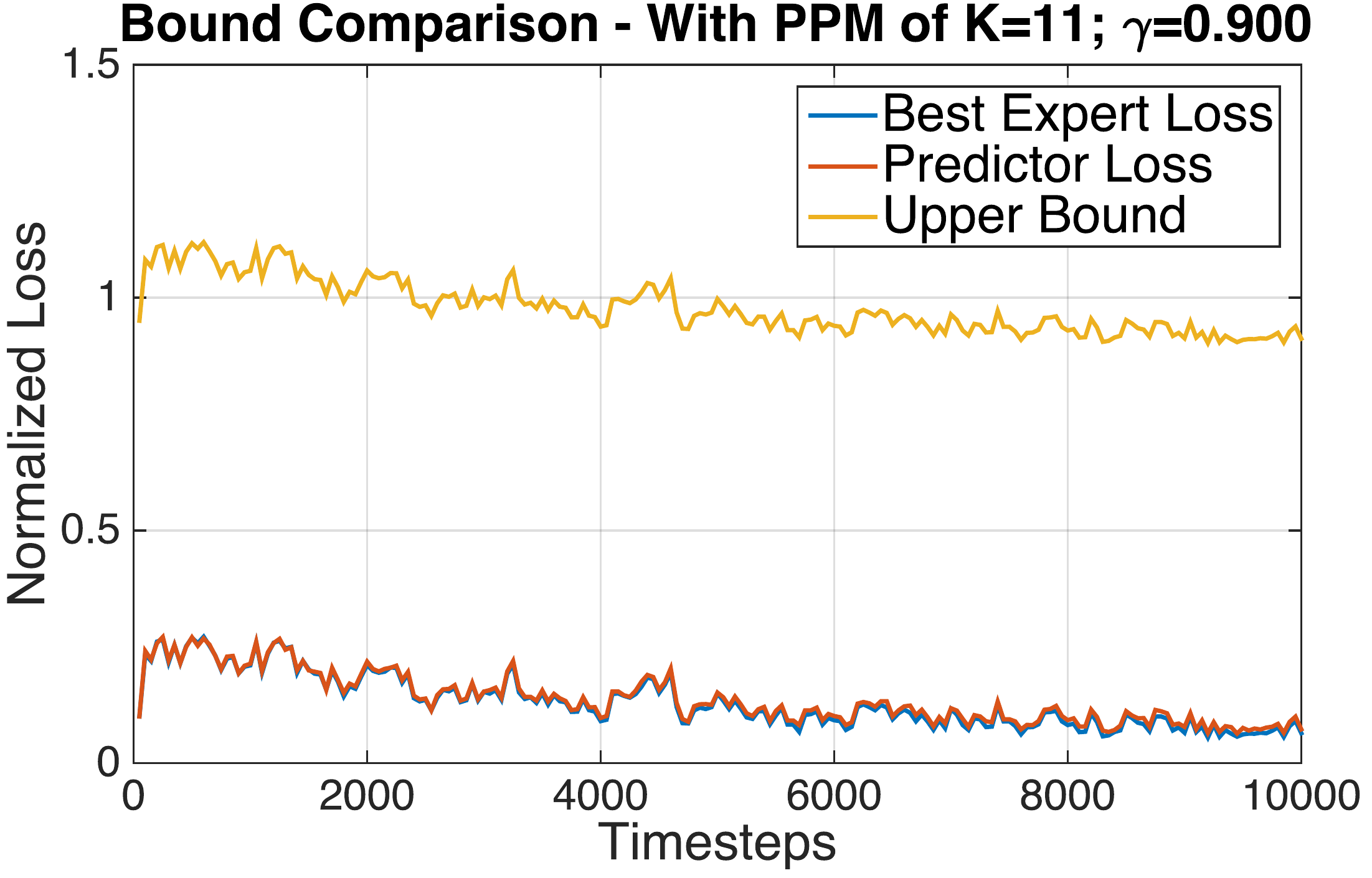}
                    \caption{No.of Experts = 11}
                    \label{fig:dscSim_g090_k1}
                \end{subfigure}%
                \caption{Loss of Predictor with $\gamma = 0.90$}\label{fig:bound_g090}
            \end{figure*}
            \begin{figure*}[t]
                \begin{subfigure}{0.33\textwidth}
                    \includegraphics[width=\linewidth]{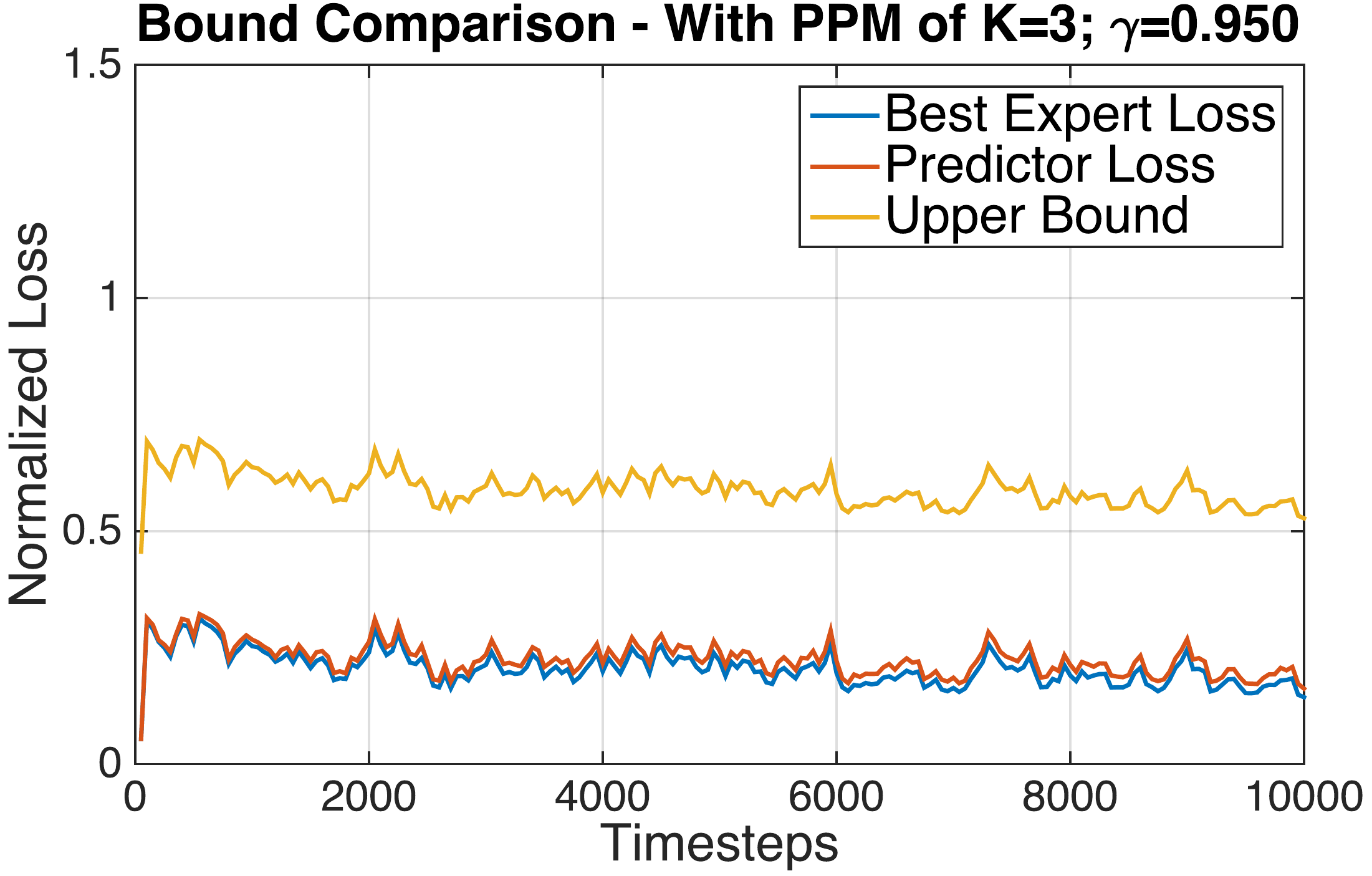}
                    \caption{No.of Experts = 3}
                    \label{fig:dscSim_g095_k03}
                \end{subfigure}%
                \begin{subfigure}{0.33\textwidth}
                    \includegraphics[width=\linewidth]{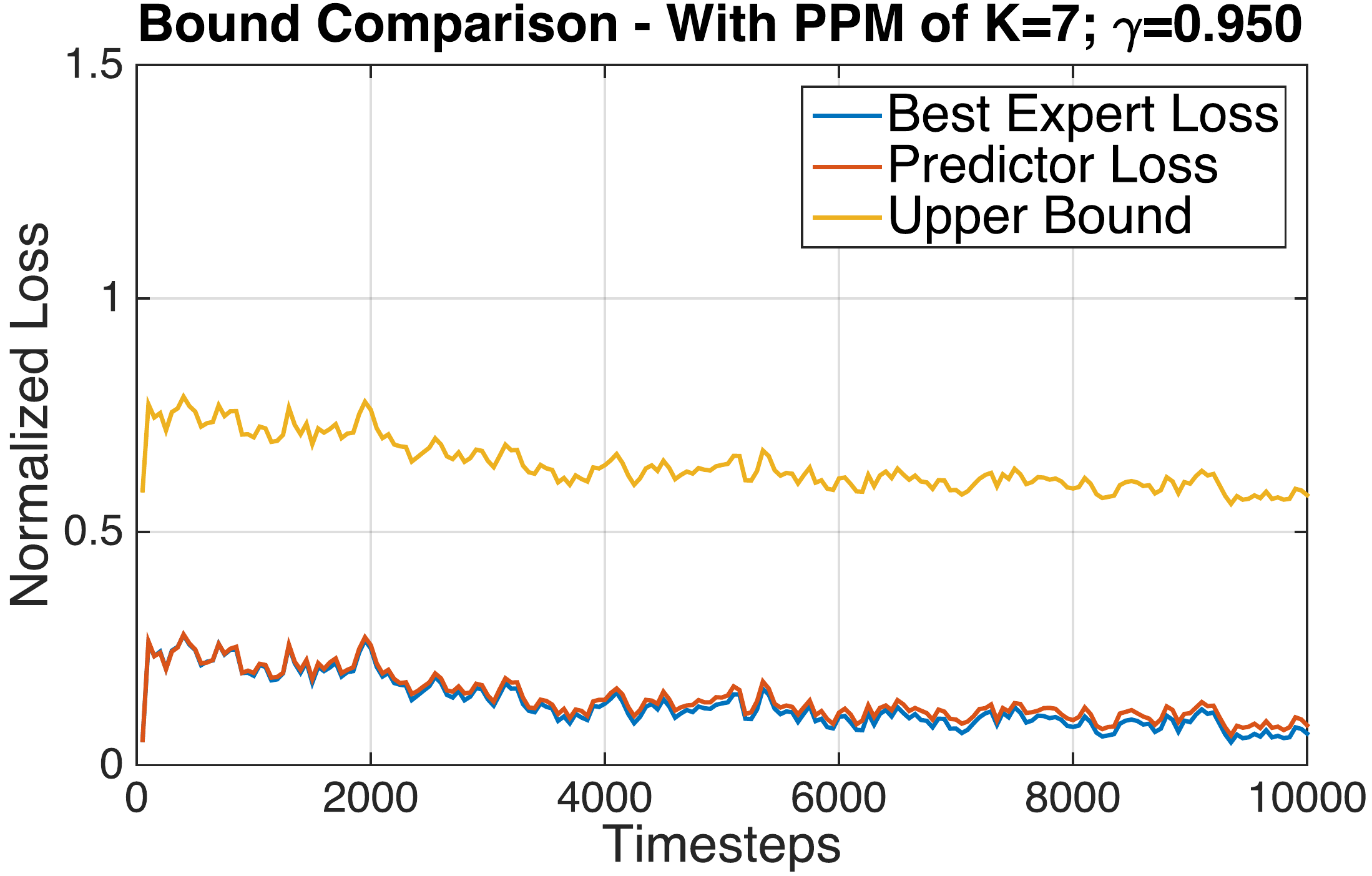}
                    \caption{No.of Experts = 7}
                    \label{fig:dscSim_g095_k07}
                \end{subfigure}%
                \begin{subfigure}{0.33\textwidth}
                    \includegraphics[width=\linewidth]{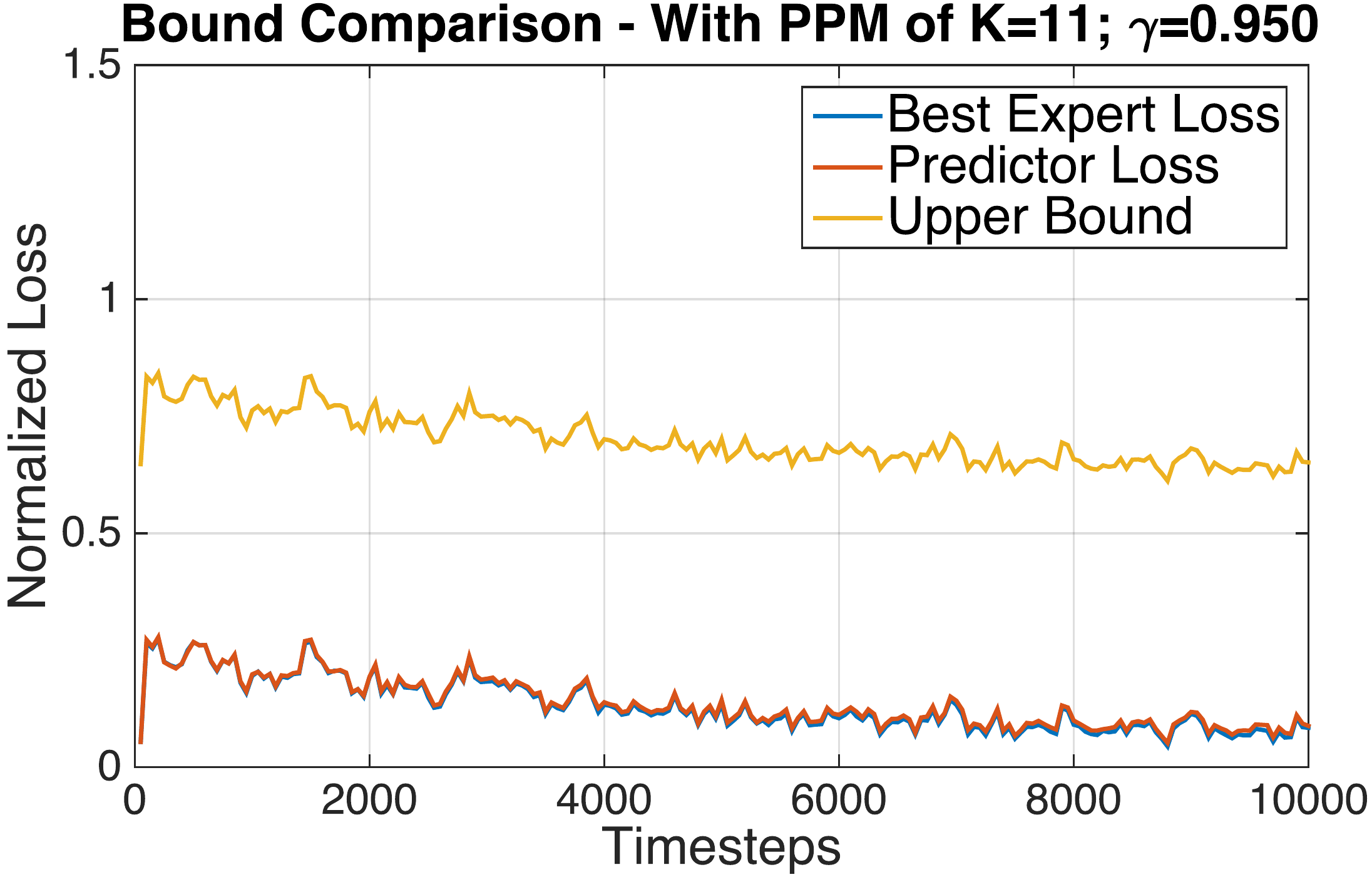}
                    \caption{No.of Experts = 11}
                    \label{fig:dscSim_g095_k1}
                \end{subfigure}%
                \caption{Loss of Predictor with $\gamma = 0.95$}\label{fig:bound_g095}
            \end{figure*} 
            \begin{figure*}[t]
                \begin{subfigure}{0.33\textwidth}
                    \includegraphics[width=\linewidth]{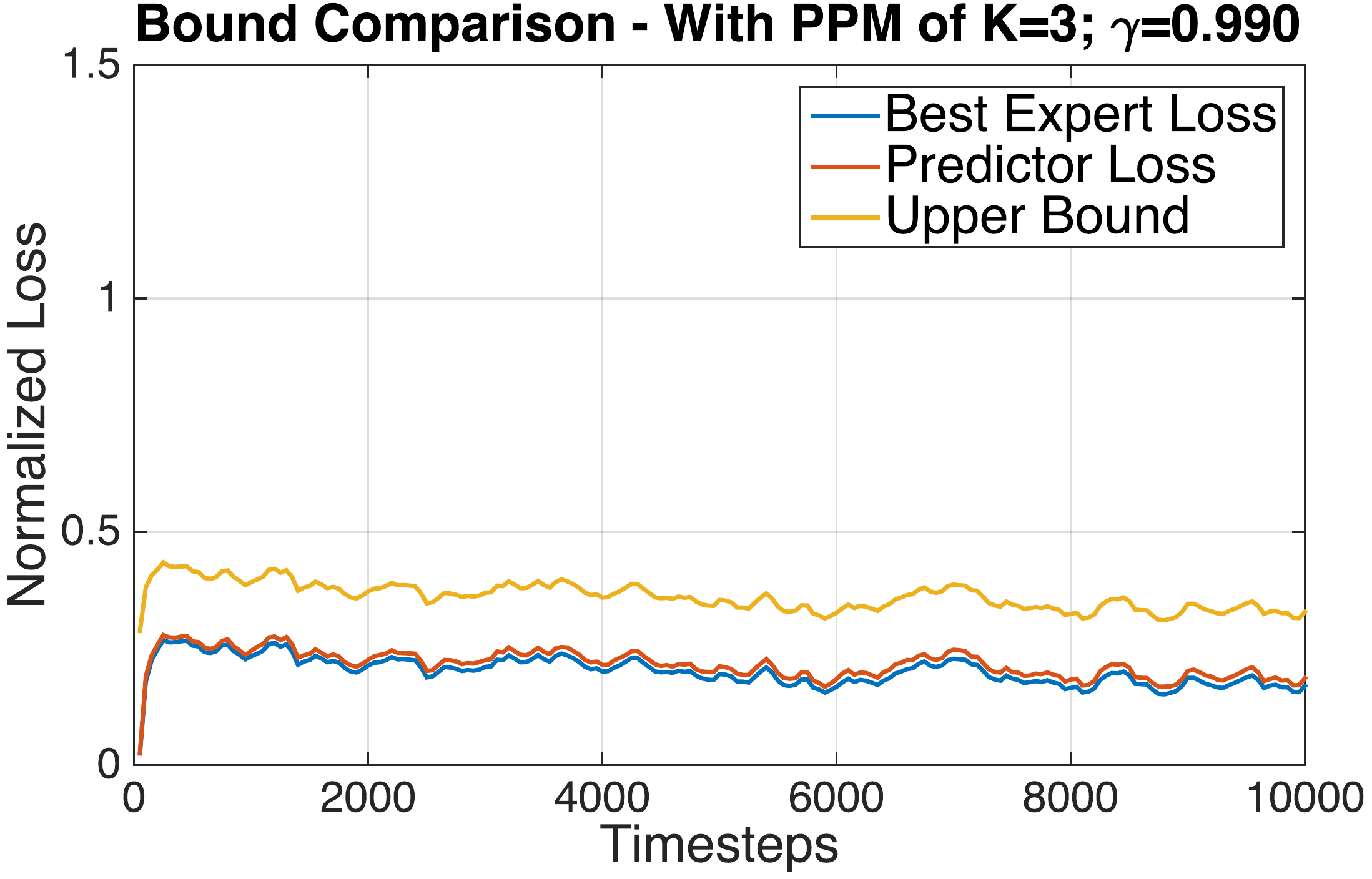}
                    \caption{No.of Experts = 3}
                    \label{fig:dscSim_g099_k03}
                \end{subfigure}%
                \begin{subfigure}{0.33\textwidth}
                    \includegraphics[width=\linewidth]{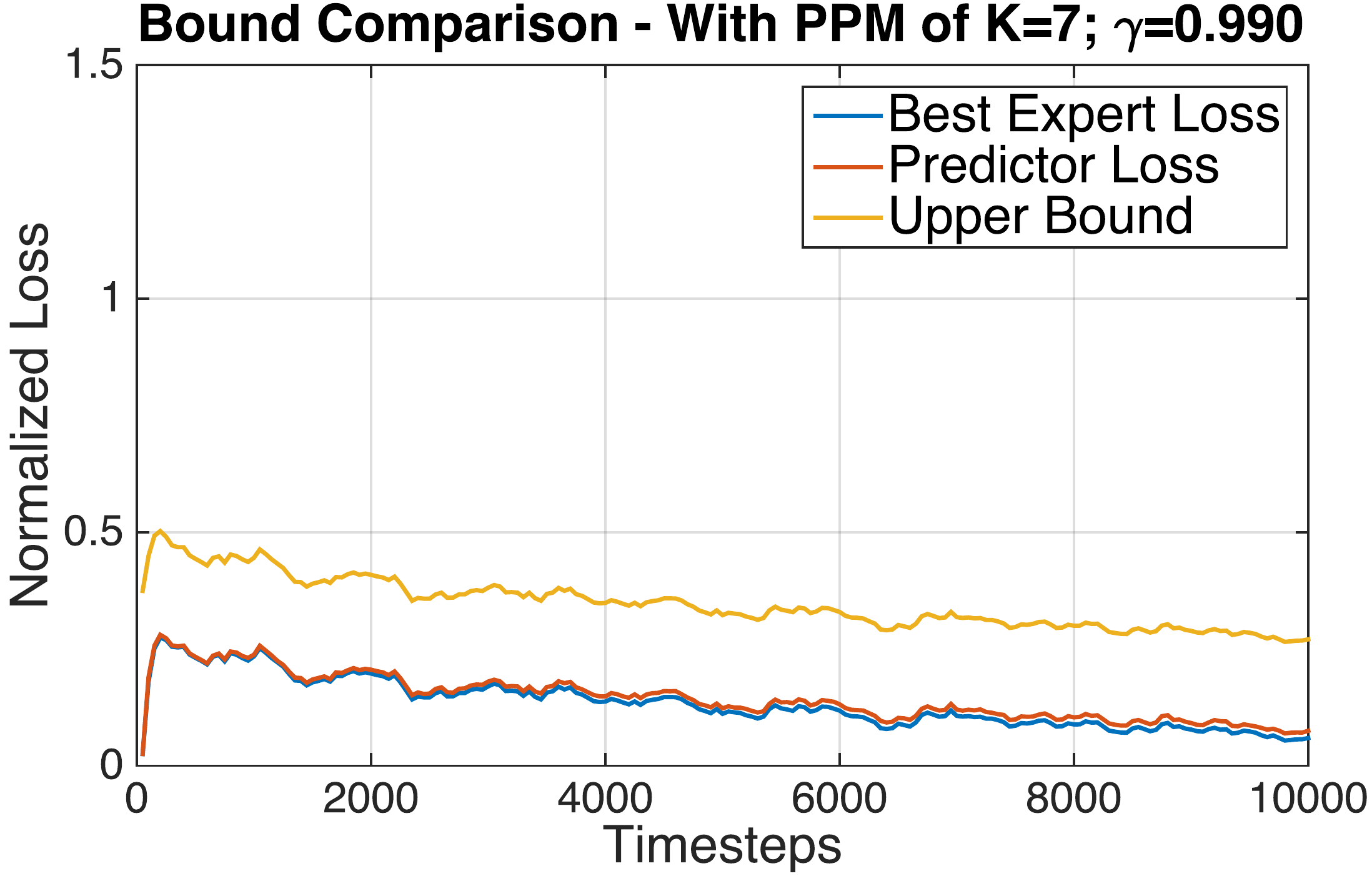}
                    \caption{No.of Experts = 7}
                    \label{fig:dscSim_g099_k07}
                \end{subfigure}%
                \begin{subfigure}{0.33\textwidth}
                    \includegraphics[width=\linewidth]{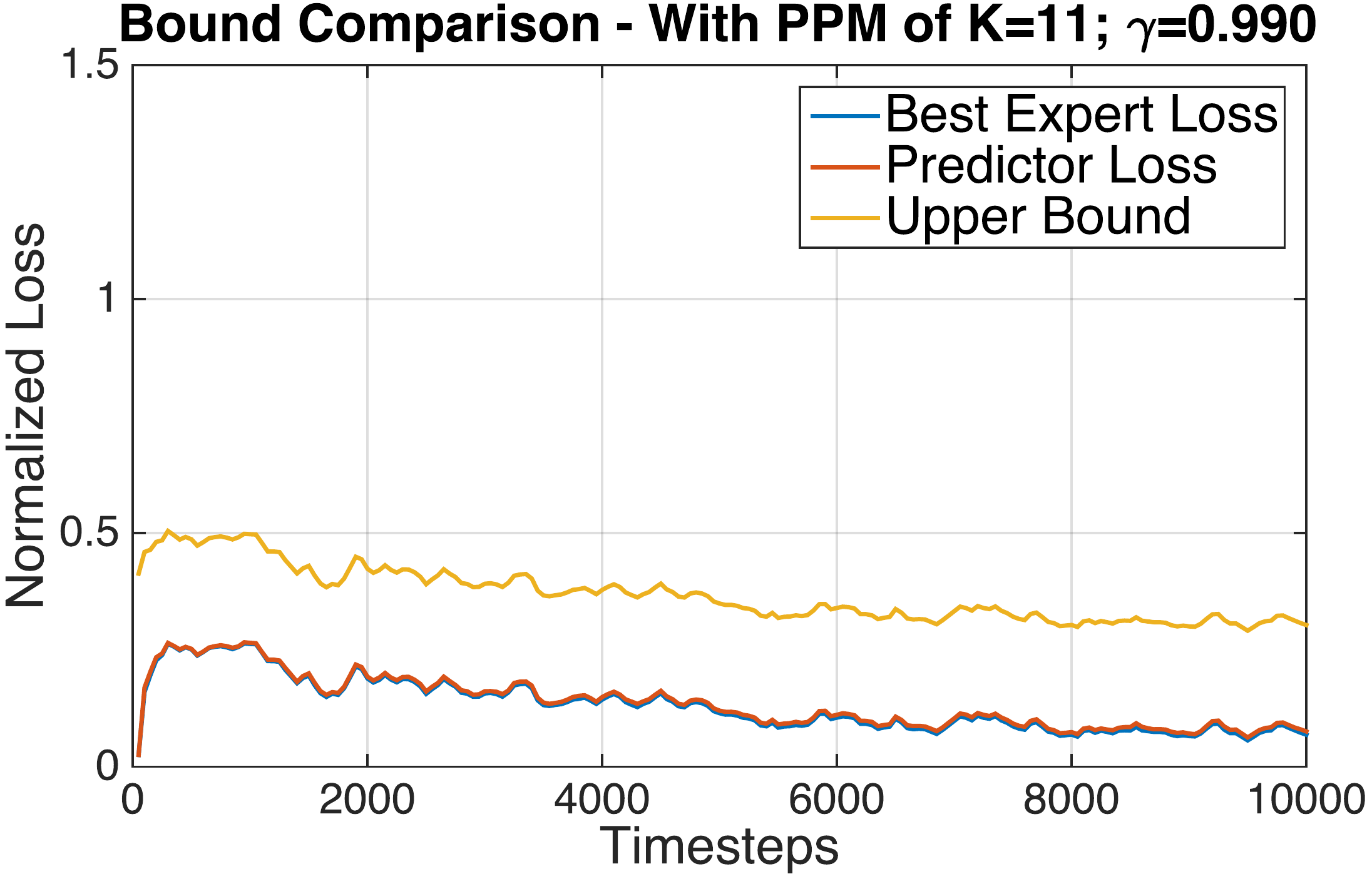}
                    \caption{No.of Experts = 11}
                    \label{fig:dscSim_g099_k1}
                \end{subfigure}%
                \caption{Loss of Predictor with $\gamma = 0.99$}\label{fig:bound_g099}
            \end{figure*}
            \begin{figure*}
                \begin{subfigure}[b]{0.33\textwidth}
                    \includegraphics[width=\linewidth]{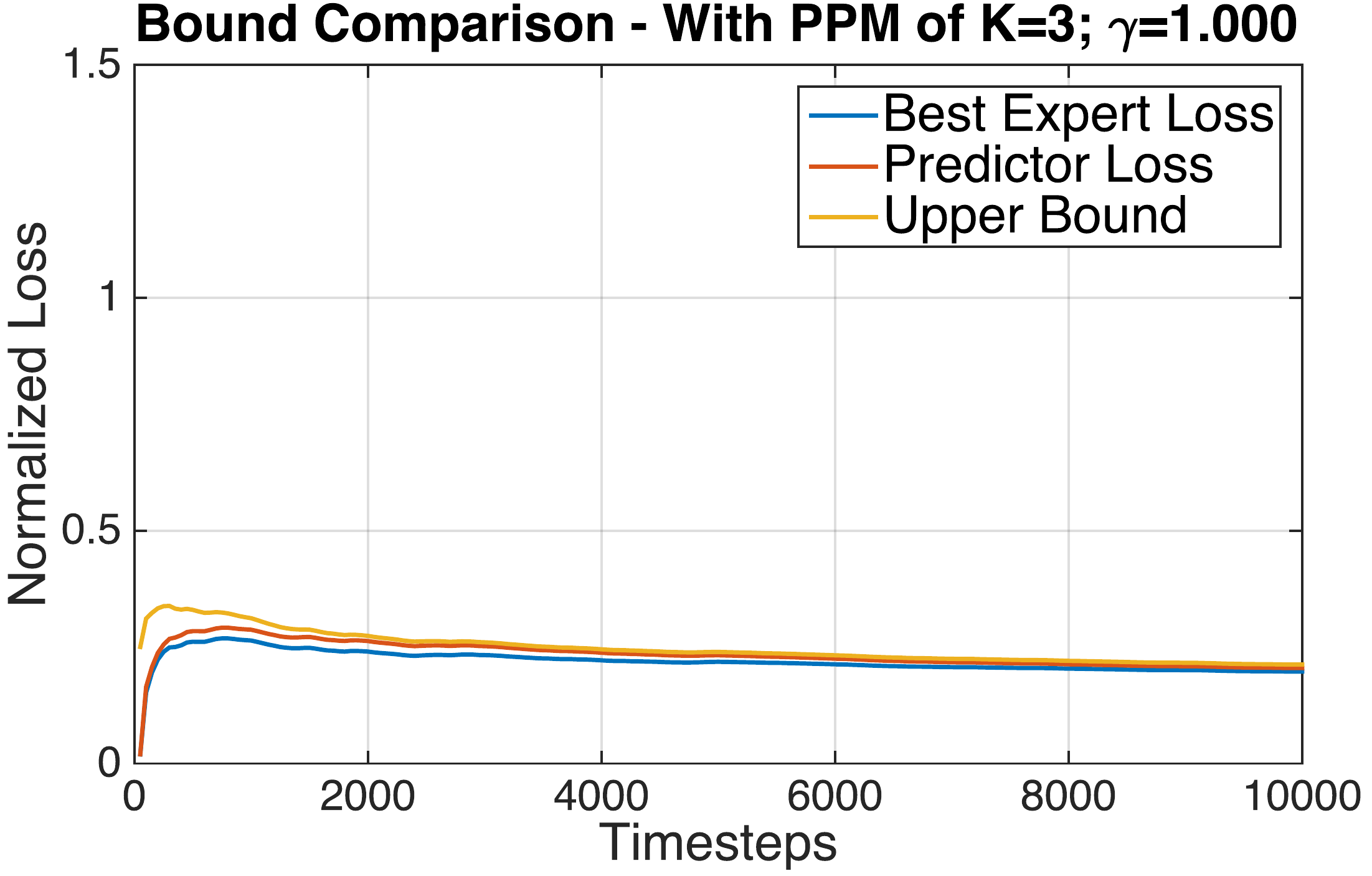}
                    \caption{No.of Experts = 3}
                    \label{fig:dscSim_g100_k03}
                \end{subfigure}%
                \begin{subfigure}[b]{0.33\textwidth}
                    \includegraphics[width=\linewidth]{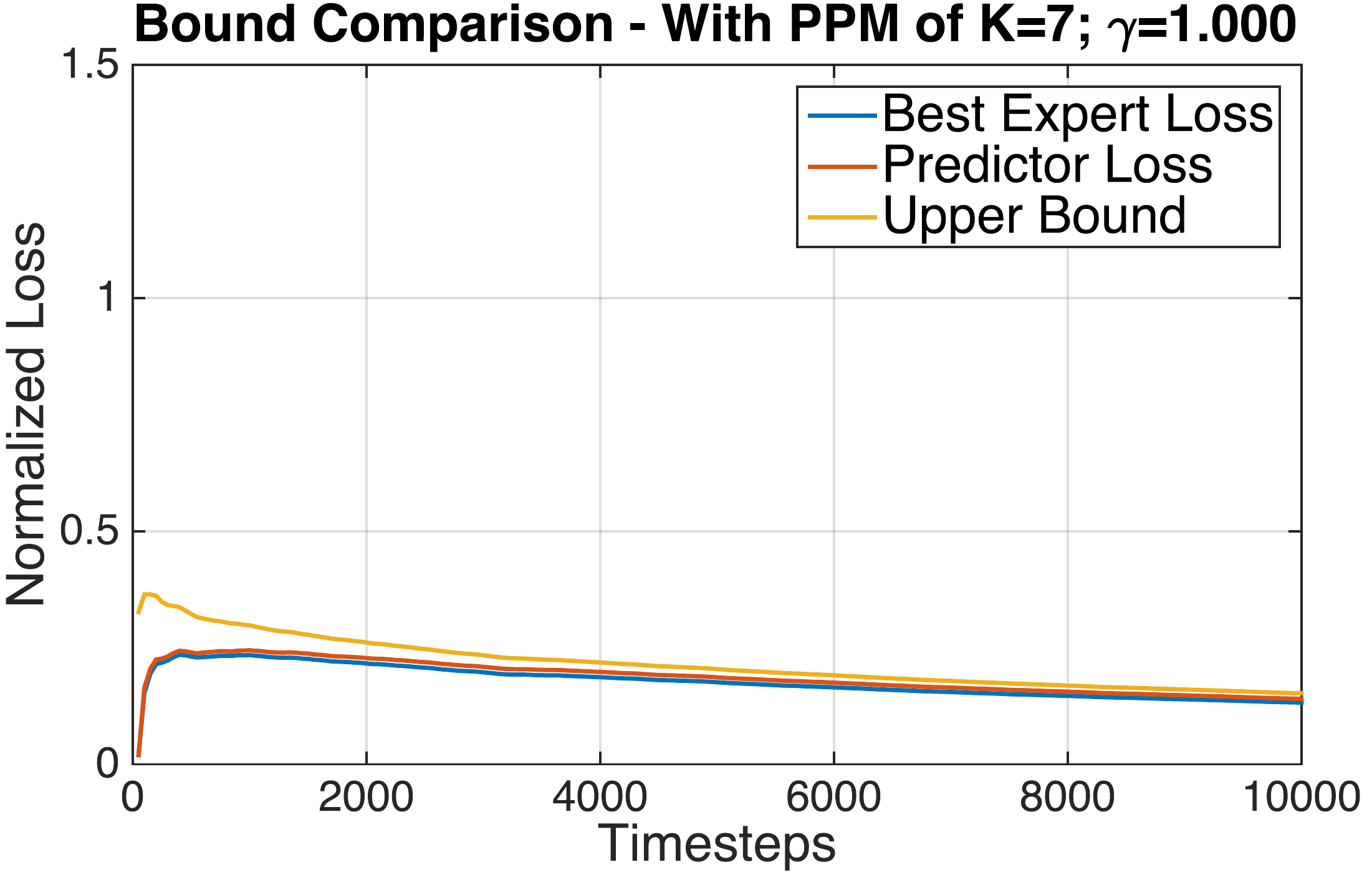}
                    \caption{No.of Experts = 7}
                    \label{fig:dscSim_g100_k07}
                \end{subfigure}%
                \begin{subfigure}[b]{0.33\textwidth}
                    \includegraphics[width=\linewidth]{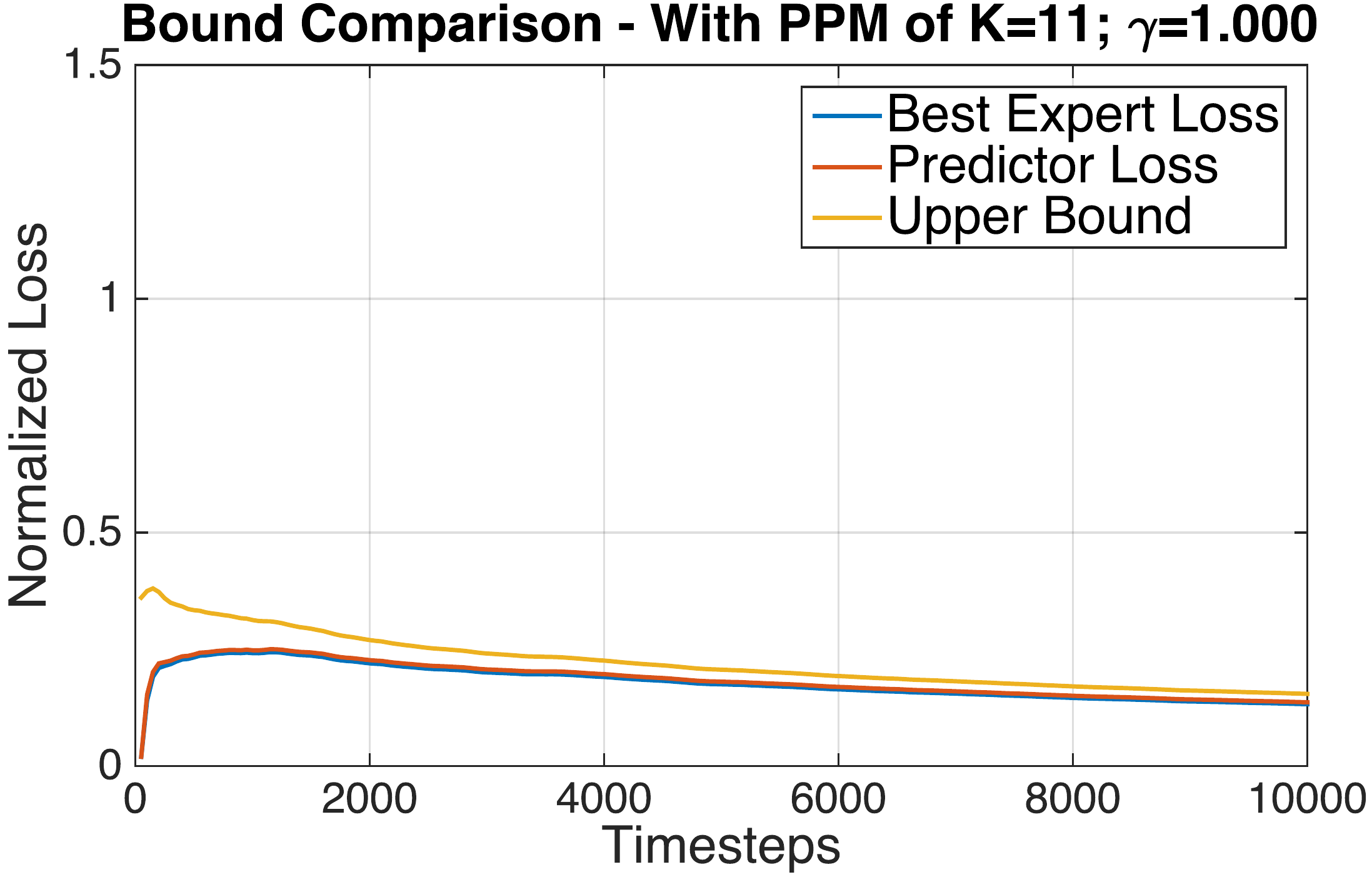}
                    \caption{No.of Experts = 11}
                    \label{fig:dscSim_g100_k1}
                \end{subfigure}%
                \caption{Loss of Predictor with $\gamma = 1.00$. This is equivalent to HEGDE.}\label{fig:bound_g100}
            \end{figure*}

            To verify the validity of the bound derived in Theorem 5,
	        experiments are done in a synthetic dataset and the results are reported after averaging
	        over multiple runs. To generate test sequences for this experiment a Markov model is
	        created as follows. Sequence of User\#0 from the Cognitive Assessment dataset of CASAS
	        project (See\cite{casas_wsu}) is taken and first 5000 symbols are trained into a trie with context length $6$.
	        Then this trie is used to generate sequences which are in turn fed into Discounted
	        HEDGE with PPM for the prediction task. This is repeated over 100 independent runs 
	        and the average results are reported.	        
	        
	        Experiments are conducted in three different scenarios for four different values of
	        $\gamma$. In each experiment, we varied the context lengths available in the pool
	        of experts. Since the synthetic data is coming from a model having context length 6,
	        we tried with context length of 2, 6 and 10. This in turn results in 3, 7 and
	        11 experts in the pool respectively. In all the experiments, optimal value of $\beta$
	        is calculated as mentioned in Eqn. 13 in main paper. Results are shown in Fig
	        \ref{fig:bound_g090} - Fig \ref{fig:bound_g100}. 
	        Loss is normalized by dividing by
	        $\tilde{L}$ to facilitate a direct comparison results with different $\gamma$ values.
            
            From the above results, we can observe that the proposed bound holds for in all the
            experimental scenarios. Loss are normalized to $\tilde{L}$, i.e., the fraction of 
            observed loss to the maximum attainable loss is plotted. 

%
\end{document}